\documentclass[acmsmall,screen]{acmart}

\setcopyright{rightsretained}
\acmPrice{}
\acmDOI{10.1145/3563325}
\acmYear{2022}
\copyrightyear{2022}
\acmSubmissionID{oopslab22main-p353-p}
\acmJournal{PACMPL}
\acmVolume{6}
\acmNumber{OOPSLA2}
\acmArticle{162}
\acmMonth{10}

\usepackage{backnaur}
\usepackage{booktabs}   
\usepackage{subcaption} 

\usepackage{outlines}

\usepackage{verbatim}

\usepackage{accents}

\usepackage{macros}
\usepackage{mathtools, nccmath}

\usepackage{wrapfig}
\usepackage{paralist}

\usepackage{hyperref}

\usepackage[export]{adjustbox}
\makeatletter
\def\thm@space@setup{%
  \thm@preskip=5cm plus 1cm minus 2cm
  \thm@postskip=\thm@preskip 
}
\makeatother

\newif\ifextended
\extendedtrue

\newcommand{\extVersion}[1]{\ifextended #1 \fi}

\begin{document}

\title[]{Scalable Verification of GNN-Based Job Schedulers}         


\author{Haoze Wu}
\affiliation{
  \department{Department of Computer Science}              
  \institution{Stanford University}            
  \country{USA}                    
}

\author{Clark Barrett}
\affiliation{
  \department{Department of Computer Science}              
  \institution{Stanford University}            
  \country{USA}                    
}

\author{Mahmood Sharif}
\affiliation{
  \department{School of Computer Science}              
  \institution{Tel Aviv University}            
  \country{Israel}                    
}

\author{Nina Narodytska}
\affiliation{
  \institution{VMware Research}            
  \country{USA}                    
}

\author{Gagandeep Singh}
\affiliation{
  \department{School of Computer Science}              
  \institution{University of Illinois at Urbana-Champaign}            
  \country{USA}                    
}


\begin{abstract}
Recently, Graph Neural Networks (GNNs) have been applied for scheduling jobs over clusters, achieving better performance than hand-crafted heuristics. Despite their impressive performance, concerns remain over whether these GNN-based job schedulers meet users' expectations about other important properties, such as strategy-proofness, sharing incentive, and stability. 
In this work, we consider formal verification of GNN-based job schedulers. We address several domain-specific challenges such as networks that are deeper and specifications that are richer than those encountered when verifying image and NLP classifiers. We develop \sys, the first general framework for verifying both single-step and multi-step properties of these schedulers based on carefully designed algorithms that combine abstractions, refinements, solvers, and proof transfer. Our experimental results show that \sys achieves significant speed-up when verifying important properties of a state-of-the-art GNN-based scheduler compared to previous methods.
\end{abstract}


\maketitle

\section{Introduction}
\label{sec:intro}



Designing efficient job scheduling for multi-user distributed-computing clusters is a challenging and important task~\cite{barroso2013datacenter}. One of the main evaluation metrics of a schedule is performance, for example optimizing job completion time on a job profile. However, the user expectation typically requires that the scheduler satisfy a number of important properties beyond performance, such as strategy-proofness, sharing incentive, and stability~\cite{fairness,locality}.  If a scheduler lacks any of these properties, the result could be catastrophic, potentially costing millions of dollars at scale.  For example, if the scheduler is not strategy-proof (meaning that users can benefit from misrepresenting their job attributes), the users would be incentivized to manipulate their jobs to get them scheduled earlier than they are supposed to. The result could be long waiting times for all users or inefficient overall operation of the cluster~\cite{locality}.

Recently, a class of job schedulers \cite{decima,park2021learning,sun2021deepweave} based on Graph Neural Networks (GNNs) were shown to achieve significant performance improvement over schedulers using hand-crafted heuristics. However, whether these GNN-based job schedulers possess essential properties is not known and, more importantly, there are no tools available to check whether these properties hold.
Formally guaranteeing these properties is known to be difficult and until now has only been achieved for simple hand-crafted policies~\cite{shenker2013choosy}.
Introducing techniques for proving or disproving such properties for GNN-based schedulers would allow system designers to better evaluate the policies implemented by these schedulers, making adjustments so that the scheduler satisfies users' expectations without sacrificing the performance too much.
However, GNNs' decision-making processes are complex and opaque, making it challenging to formally validate these properties.



In this work, we focus on the formal verification of GNN-based job schedulers, which, to the best of our knowledge, has not been considered in prior work. In particular, given a specification over a GNN-based job scheduler, our goal is to either formally prove the specification holds or disprove it with a counter-example.
While there is a growing body of work on formally analyzing and verifying properties of deep neural networks applied in the vision, robotics, and natural language processing domains, work on formal analysis of ML models in the systems domain has been limited. 
This may be explained in part by the unique challenges posed by the systems domain, some of which we outline below.

\paragraph{Computation graph with 100+ layers.} GNN-based systems, including schedulers, perform a message-passing algorithm as part of the inference stage. 
While message passing can be unrolled into a sequence of affine and non-linear activation layers, the resulting network is quite deep, typically containing over $100$ layers. Existing state-of-the-art verifiers are designed to handle shallower networks (typically $<20$ layers) and begin to lose substantial precision~\cite{zelazny2022optimizing,nnv,planet,dlv,deeppoly,kpoly,frown,crown,wong,reluval,neurify,sherlock,mipverify,fastlin,barrier,barrier-revisited,sdp,star,AI2,bcrown,singh2019boosting,cnn-cert,deepz,wu2020parallelization,gpupoly} or scalability ~\cite{unified,reluplex,marabou,planet,nnenum,nnv,rpm,deepsplit,fromherz2020fast,optAndAbs,mipverify,mip01,branching,babsr,peregrinn,dependency,xu2020fast,wu2022efficient} with increasing network depth. To deal with this challenge, we propose a new, general framework for iterative forward-backward abstraction refinement that balances the analysis precision and speed for deeper networks.

\paragraph{Rich specifications.}
Many desirable properties require reasoning about sets of nodes rather than a single node. For example, one might specify ``no task from job A is scheduled before at least one task from job B is finished'' for strategy-proofness. Properties like this contain a large disjunction: we need to check the requirement for each task in job A. We propose an abstraction technique that can reason about multiple disjuncts simultaneously to speed up verification of such properties. As with robustness properties, many desirable properties for schedulers can be defined both globally and locally. We focus on the latter which is popular in the neural network verification literature and stronger than empirical evaluation on finite sets of inputs as done in the past for more complex scheduling policies~\cite{kandasamy2020online}. Our framework can theoretically also handle global properties. We discuss the practical difficulties of it in Sec.~\ref{sec:spec} and leave it as future work.

\paragraph{Sequential decision making.} 
Schedulers perform sequential decisions to schedule tasks. Therefore, to thoroughly analyze the learned schedulers, it is not sufficient to only reason about single-step input-output properties considered by state-of-the-art verifiers, as a malicious user can craft a job that affects the scheduler's behavior downstream. Therefore, we consider multi-step verification to reason about bounded traces produced by a sequence of scheduling actions from the scheduler. This adds an extra layer of complexity on top of the already challenging single-step verification, as we need to reason about all the different states along different traces, which requires unrolling the system. 
To address this, we introduce a proof-transfer encoding of the system which only registers incremental changes in the network encoding along the traces. This significantly speeds up complete verification in the multi-step setting.


\paragraph{This work} We present the first approach for formally analyzing state-based and trace-based properties of GNN-based job schedulers. We build general algorithms for single-step and multi-step verification. 
Our main contributions are:
\begin{outline}

\1 We present a new, generic iterative refinement framework for forward-backward analysis of neural networks. Our framework can be instantiated with popular numerical domains such as Zonotope or DeepPoly to iteratively refine the analysis results.
\1 We provide a novel, tunable node abstraction for a set of node embeddings produced by the GNN to speed up the verification of properties with multiple disjuncts.
\1  
We present an algorithm for multi-step verification based on trace enumeration. To improve speed, we leverage proof transfer to reuse encodings for the parts of the GNN structures that do not change across time steps. 
\1 We provide an end-to-end implementation of our approach in a framework called \sys

(\textcolor{red}{ve}rification of \textcolor{red}{G}NN-b\textcolor{red}{a}sed \textcolor{red}{s}chedulers) and evaluate its effectiveness for checking desirable state-based and trace-based reachability properties for the state-of-the-art GNN-based scheduler Decima~\cite{decima}. 
Our results show that analysis with \sys is significantly more precise and scalable than baselines based on existing state-of-the-art verifiers. 
Using \sys we prove that Decima satisfies the strategy-proofness property in many cases but not always. Thus adjustments in the training procedure are potentially needed to make Decima fully strategy-proof.
Our system and benchmarks are available at \href{https://github.com/anwu1219/vegas}{this repository}.
\end{outline}

\section{Overview}
\label{sec:overview}

\begin{figure*}[t]
    \centering
    \includegraphics[width=0.98\textwidth]{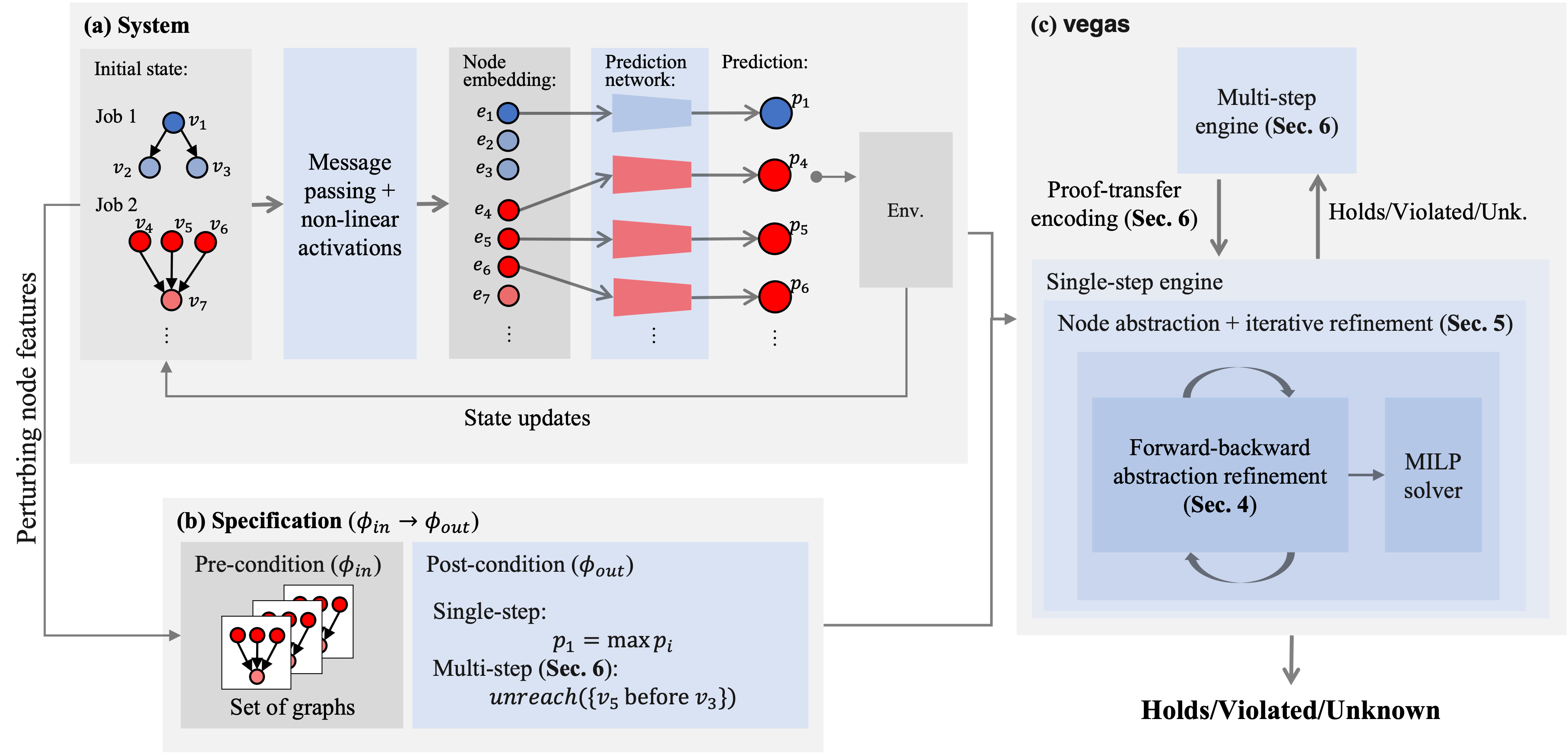}
    \caption{Overview of our verification workflow. It has three main components: (a) the system to verify; (b) a formal language for specifying properties; and (c) the verification engine \sys.}
    \label{fig:overview}
\end{figure*}

In this section, we first describe our verification workflow and then explain our key technical contributions using small intuitive examples. Formal details are in later sections.

\subsection{Verification Workflow}

Our verification workflow shown in \figref{fig:overview} has three components: (a) the system to verify; (b) a formal language for specifying properties; and (c) the verification engine \sys.

\paragraph{GNN-based scheduling system.}
GNN arises as a natural solution for learning-based job schedulers on clusters because many clusters (e.g., Spark) encode jobs as directed acyclic graphs (DAGs), with each node representing a computational \emph{stage} consisting of one or more tasks that can be run in parallel. Each node is associated with a feature vector \extVersion{(described in \appref{app:features}) }containing all the state information for that node, including the average task duration and the number of remaining tasks. There is an edge from stage $\node_i$ to stage $\node_j$ if the latter takes the outputs of the former as inputs. That is, $\node_j$ cannot be scheduled before $\node_i$ is completed. We call a node with no children a \emph{frontier node}. The input to the GNN-based scheduler is a set of jobs to schedule. 
The output of the GNN is a score $\out_i$ for each frontier node $\node_i$, representing the estimated reward if $\node_i$ is scheduled next. The node with the highest score is selected to be scheduled.

As is typical in a GNN, the GNN-based job scheduler contains a message passing component which computes a latent representation (i.e., an embedding) $e_i$ for each node $v_i$.
We define message passing precisely in \secref{subsec:gcn}.
The score $\out_i$ for a frontier node $\node_i$ is computed from a prediction network which takes $e_i$ as input.
The scheduling action (i.e., the node with the highest score) is reported to the environment, which schedules the reported node and produces a new state (with, for example, nodes removed).

\paragraph{Specifications.}
We consider a wide range of specifications of the form $\phi_{in} \to \phi_{out}$. 
We assume $\phi_{in}$ is a conjunction of linear constraints over the GNN input features, and consider post conditions $\phi_{out}$ in both single-step and multi-step settings. Single step post conditions are logical constraints over linear inequalities on the outputs of the network. Multi-step post conditions are defined in terms of unreachability of ``bad'' \emph{traces} (i.e., sequences of scheduling decisions).


\paragraph{Verifier.}
Our verification engine \sys has two main components, a single-step engine and a multi-step engine. 
Motivated by the unique challenges in this verification setting, we propose a forward-backward abstraction refinement framework (\secref{sec:backward}) which goes beyond the forward-propagation-only abstract interpretation, as well as a node-abstraction scheme (\secref{sec:abtraction}) for handling disjunctions in the verification query.
The multi-step engine runs a trace enumeration procedure that repeatedly invokes the single-step engine. We propose an efficient encoding of the unrolled system referred to as the \emph{proof-transfer encoding} (\secref{sec:multi}) which significantly reduces the run-time.


\subsection{Forward-backward abstraction refinement}

As mentioned earlier, one of the distinctive features in GNN verification is the need to reason about very deep computational graphs resulting from the unrolling of the message-passing procedure. Forward abstract interpretation techniques are less effective here as the imprecision can grow exponentially with increasing depth of the computation graph. We propose to refine the forward abstraction by backward refinement guided by the output constraints. This yields a general forward-backward abstraction refinement loop (\secref{sec:backward}). While our ideas are driven to tackle challenges in GNN verification, we formalize and implement the proposed technique in a general manner so that it can be applied for neural networks with different architectures and activations.


\paragraph{Running example.} We illustrate the forward-backward abstraction refinement on a pre-trained fully-connected feed-forward neural network with Leaky ReLU activation functions ($\sigma(x) = \max(\alpha x, x)$) shown in \figref{fig:backToyExample}. Here $\alpha$ is a hyper-parameter of Leaky ReLU. For numerical simplicity, we assume $\alpha$ = 0.1 in this example. 
We use Leaky ReLU as an example since it is used in the state-of-the-art GNN-based job scheduler Decima which we set out to verify. 
Note that while the running example uses a feed-forward neural network for simplicity, in practice the architecture of a GNN is much more complex (e.g., contains residual connections). We discuss how to handle forward-backward analysis in the GNN setting in Sec.~\ref{subsec:skip}. 

The network here consists of four layers: an input layer, two hidden layers, and an output layer with two neurons each. The outputs of a (non-input) layer are computed by applying an affine transformation to the last layer's outputs followed by the activation function. The activation function is often not applied at the output layer (also in this example). The values on the edges represent the learned weights of the affine transformations. The values above or below the neurons represent the learned biases (translation values) of the affine transformation. 
For example, the top neuron in the first hidden layer, $x_4$, can be computed as $\sigma(x_2)$, where $x_2$, the pre-activation value of the neuron, is equal to $x_0 + x_1$. 

\paragraph{Specification.}
Let us assume a hypothetical job profile with two disconnected nodes. Suppose their feature vectors (1-dimensional in this case) range from $[0, 1]$. Our goal is to prove that the score for the second node ($x_{11}$) is always greater than the score for the first node ($x_{10}$), for any possible values of the two features in the range $[0, 1]$.

\paragraph{Forward abstract interpretation.}
A typical abstract interpretation on neural networks \cite{crown,deepz,deeppoly,kpoly} involves propagating the input set (represented in pre-defined abstraction domain such as Zonotope or DeepPoly) forward layer by layer (via pre-defined abstract transformers) to compute an over-approximation of the reachable output set. The specification holds if the over-approximated output set is disjoint from the bad states ($\neg\phi_{out}$). 
In this work, we use the DeepPoly domain~\cite{deeppoly} for forward abstract interpretation, though the refinement technique applies to any sub-polyhedral abstract domain~\cite{SinghPV:18}. We show the intervals derived by the DeepPoly analysis in the bottom (blue) box in \figref{fig:backToyExample}. The steps to deriving these bounds are omitted here\extVersion{ (can be found in \appref{app:deeppoly})}. As shown in the figure, the output bounds derived by DeepPoly are not precise enough to prove the property.

\paragraph{Backward abstraction refinement.}
Normally, at this point, the abstraction-based analysis is inconclusive and we have to resort to search-based methods (e.g. MILP or SMT solvers) which perform case analysis on the Leaky ReLUs and have an exponential runtime in the worst case. Instead, we observe that the forward abstract interpretation typically ignores the post condition when computing the over-approximation at each layer leaving room for refinement guided by the verification conditions and proving the property without invoking search-based methods. In the case where a complete search is unavoidable, a refined over-approximation still helps to prune the search space. 

\begin{figure*}[!t]
    \centering
    \includegraphics[width=0.98\textwidth]{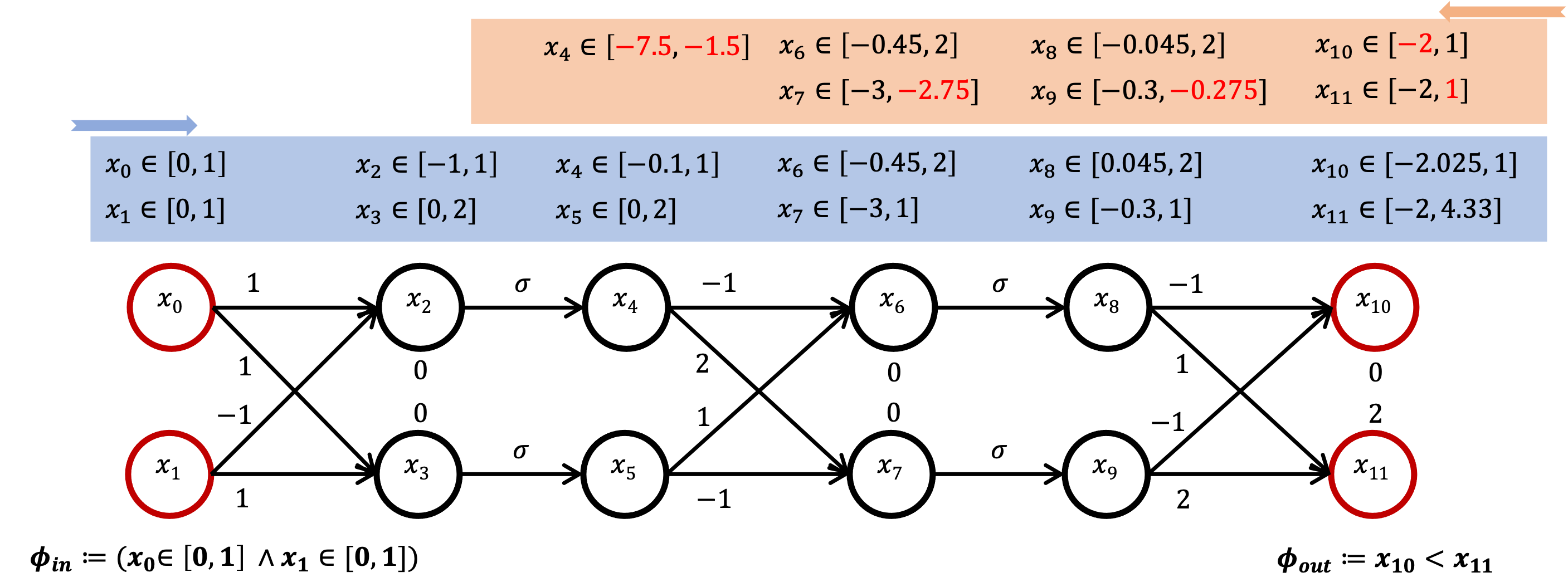}
    \caption{A toy example for forward-backward abstraction refinement.}
    \label{fig:backToyExample}
\end{figure*}


We illustrate a forward-backward abstraction refinement on our example. We associate two abstract elements, a forward one $a_i$  and a backward one $\abst{a_i}$ from the same abstract domain (e.g., DeepPoly) with each neuron $x_i$. The analysis alternates between forward and backward passes, which respectively update the forward abstract elements and the backward abstract elements. We construct new forward transformers that consider both the forward and the backward abstract elements (\secref{sec:backward}), yielding more precise forward analysis. The backward elements are initialized to $\top$ in the beginning, and therefore the first forward pass results in the same results as before. This is followed by a backward pass.

The backward analysis starts with the ``bad'' output set $\neg\phi_{out}:= x_{10} \geq x_{11}$, which we use to refine $\abst{a_{10}}$ and $\abst{a_{11}}$, which are currently set to $\top$. 
We first compute the bounds of $x_{10}$ and $x_{11}$, conditioned on the existing bounds and $\neg\phi_{out}$. This yields a tighter lower bound and upper bound for $x_{10}$ and $x_{11}$ respectively. These new bounds are then used to refine the underlying backward abstract element $\aBack_{11}$: $\aBack_{11} = \cond(\aFor_{11}, x_{11}\in [-2,1])$, where $\aFor_{11}$ is the forward abstract element updated from the forward pass and $\cond$ is the conditional transformer from the domain. $\aBack_{10}$ is updated similarly. 

We now move to the last hidden layer. Again, we first compute sound intervals for neurons $x_{8}$ and $x_9$ with Linear Programming (LP). For instance, to compute an upper bound for $x_9$, we can cast the following optimization problem:
\useshortskip 
\begin{align*}
&u_{9} = \max_{x_{8:11}} x_9,
\text{s.t. } x_{10} = -x_8 - x_9, x_{11} = x_8 + 2 x_9 + 2 \\
&\qquad x_{10} \in [-2, 1], x_{11} \in [-2, 1], x_{8} \in [-0.045, 2] 
\end{align*}

We obtain a tightened bound $x_9 \leq -0.275$. Importantly, now the underlying backward abstract element $\abst{a_9}$ is set to $\cond(a_9, [-0.3, -0.275])$ where the input-output relationship for the Leaky ReLU is linear (and can be captured exactly by domains like DeepPoly). The exactness significantly improves the analysis precision in the next forward pass.

Note that in the most general form (and in our implementation), two LPs per neuron are called to tighten the bounds. While this incurs overhead, the process is highly parallelizable as neurons from the same layer can be processed independently. More importantly, we observe that this tractable overhead (we prove complexity in \secref{sec:backward}) usually pays off in practice on challenging benchmarks which would otherwise require extensive search by a complete procedure.

After refining $\abst{a_8}$ and $\abst{a_9}$, we process $x_6$ and $x_7$, where $x_8 = \sigma(x_6)$ and $x_9 = \sigma(x_7)$. 
Due to the non-linearity of the activation function, a precise encoding of Leaky ReLU results in a Mixed Integer Linear Program, which is in general challenging to solve. Therefore, we encode a sound linear relaxation~\cite{planet} of the activation function. 

Next, using the same procedure, we derive that $x_{4} \in [-7.5, -1.5]$.
Note that this is disjoint from the interval derived during forward analysis. Intuitively, this means that for $\phi_{in}\land \neg\phi_{out}$ to hold, $x_4$ must be less than or equal to -1.5, while for $\phi_{in}$ to hold, $x_4$ must be between -0.1 to 1. This implies that $\neg\phi_{out}$ cannot hold for any input satisfying $\phi_{in}$, and the property is proved without the use of search-based methods. 

In the case where the backward analysis does not prove the property by itself, we could perform forward analysis again by taking the refined backward abstract elements into consideration. This could result in a tighter over-approximation of the output set compared to the first forward analysis. Performing backward analysis again from this tighter over-approximation could in turn result in further refinement. Thus, we could alternate between forward and backward analysis to keep refining the abstractions until either the property is proved or some convergence condition is met. We define this forward-backward analysis formally in Section~\ref{sec:backward}.

\subsection{Node abstraction with iterative refinement}

If the post-condition $\phi_{out}$ contains multiple conjuncts, the bad output set $\neg\phi_{out}$ becomes disjunctive. This is a common occurrence in practice as the post-condition often specifies that a set of output neurons all satisfy a certain property $P$. For instance, the simple post-condition ``node $\node_1$ is always scheduled'' can be formally specified as $\phi_{out} := \bigwedge \out_1 > \out_i$, where $\out_i$ extends over the score of all frontier nodes other than $\out_1$. In this case, the bad output set $\neg \phi_{out}$ becomes $\bigvee \out_1 \leq \out_i$. 
A na\"ive approach would analyze each disjunct individually, which becomes expensive, especially when the number of disjuncts is large. However, we observe that the special structure of a GNN used in node prediction tasks allows us to reason about multiple nodes simultaneously.

We illustrate this idea on the weaker post-condition ``node $\node_1$ has higher score than frontier nodes in job 2'' (in \figref{fig:overview}(a)), that is, $\bigwedge_{i\in\{4,5,6\}}\out_1 > \out_i$. As shown in \figref{fig:abs}, a vanilla approach would individually check for the unsatisfiability of the three formulas  $\out_1 \leq \out_4$, $\out_1 \leq \out_5$, and $\out_1 \leq \out_6$. However, we observe that in GNNs for node prediction tasks~\cite{wu2020comprehensive}, the prediction for a node $\node_i$ is often computed by applying the \emph{same} prediction network to the embedding $\embed_i$. It is therefore natural to consider an abstraction $\abst{e}$ at the embedding level which
contains all values that $e_4$, $e_5$, and $e_6$ can take. As illustrated in \figref{fig:abs}, suppose $e_4\in[-2,2]$, $e_5\in[-1,3]$, and $e_6\in[3,5]$. An abstract embedding $\abst{e}$ can then take values in $[-2,5]$. If $\out_1 > \abst{\out}$ holds, or equivalently, $\out_1 \leq \abst{\out}$ is unsatisfiable, then the original post-condition must hold. We state and prove this formally in \secref{sec:abtraction}.
On the other hand, if we obtain a spurious counter-example, the analysis is inconclusive 
\begin{wrapfigure}{r}{0.5\textwidth}
\begin{center}
\vspace{-4mm}
\includegraphics[width=0.5\textwidth]{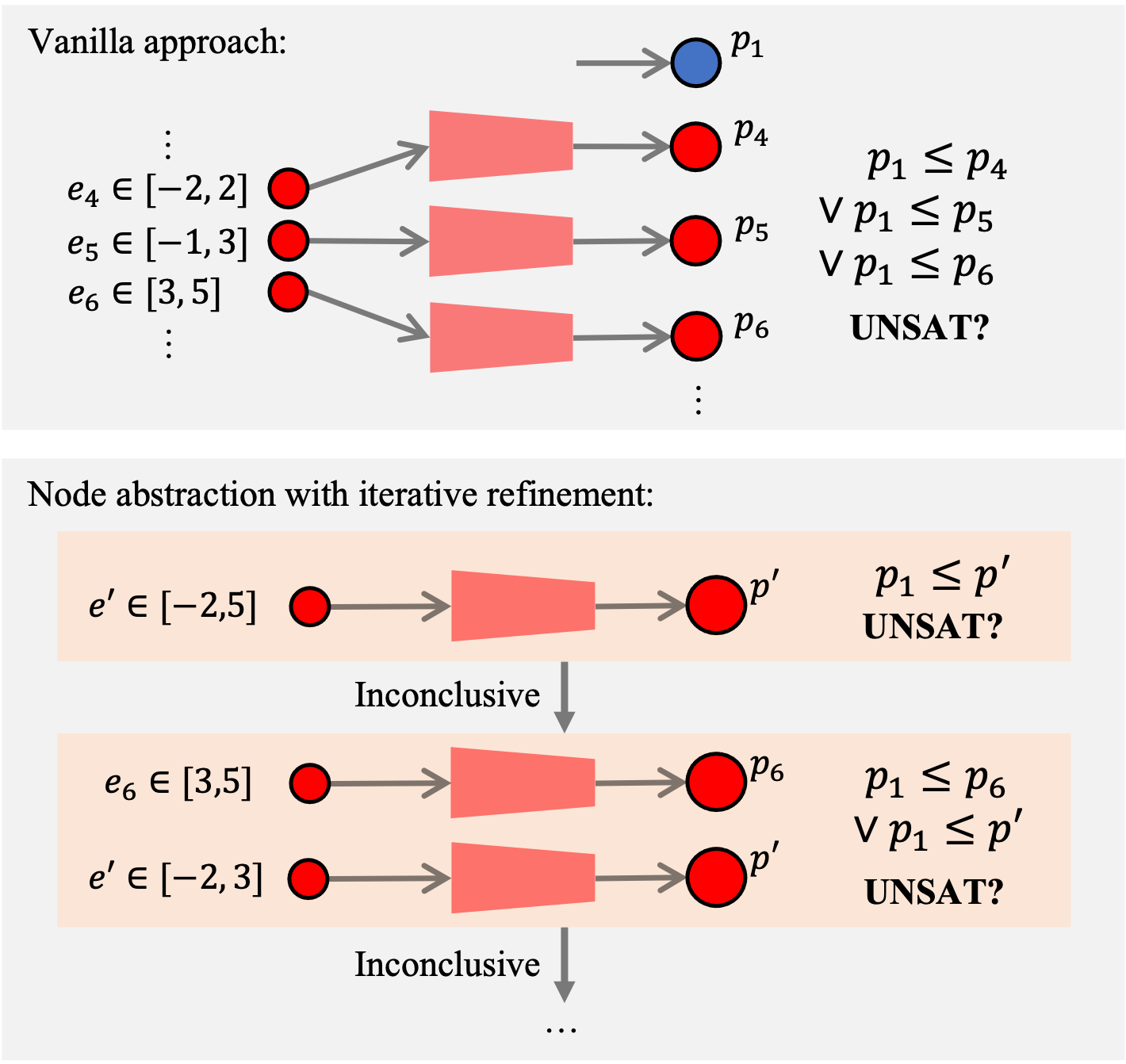}
\caption{A toy example for node abstraction.}
\vspace{-6mm}
\label{fig:abs}  
\end{center}
\end{wrapfigure}
and we need to refine 
the abstraction. 
One way to refine the abstraction is by reducing the number of node embeddings considered in the abstract embedding. In particular, we heuristically remove the embedding that reduces the volume of the abstraction the most, in this case $e_6$, and try to reason about $\out_1 > \out_6$ and $\out_1 > \abst{\out}$ individually. Now $\abst{e}$ is more constrained  ($\abst{\embed}\in [-2,3]$), and the property is more likely to hold on $\abst{\embed}$. We can perform this refinement iteratively until either the property is proved or a real counter-example is found. In the worst case, no node abstraction can be performed and we would have to examine each node individually. In practice, it often pays off to speculatively reason about multiple neurons simultaneously, especially when the number of disjuncts is large.

\subsection{Beyond single-step verification}

Building on top of our single-step engine combining forward-backward refinement and node abstraction, we present a procedure for verifying multi-step properties in the form of trace (un)reachability ($\phi_{in}\to\unreach(T)$). As we shall see, this allows us to define meaningful specifications over the job-scheduling system. We illustrate the procedure on the example in \figref{fig:multi}. The postcondition we verify states that ``$\node_5$ cannot be scheduled before $\node_3$''. 
Our procedure proves this by computing the set of feasible traces from the initial state by repeatedly invoking the single-step engine. For example, at step 0, we ask the single step engine to check $\phi_{in}\to (\node_1 \geq \node_4)$ and $\phi_{in} \to (\node_1 \leq \node_4)$. 
If both can be violated (i.e., $\node_1$ and $\node_4$ can be scheduled), we proceed by computing the reachable traces from $\node_1$ and $\node_4$, respectively, and so forth. During this process,  whenever we construct a reachable (partial)
\begin{wrapfigure}{r}{0.55\textwidth}
\begin{center}
\vspace{-4mm}
\includegraphics[width=0.55\textwidth]{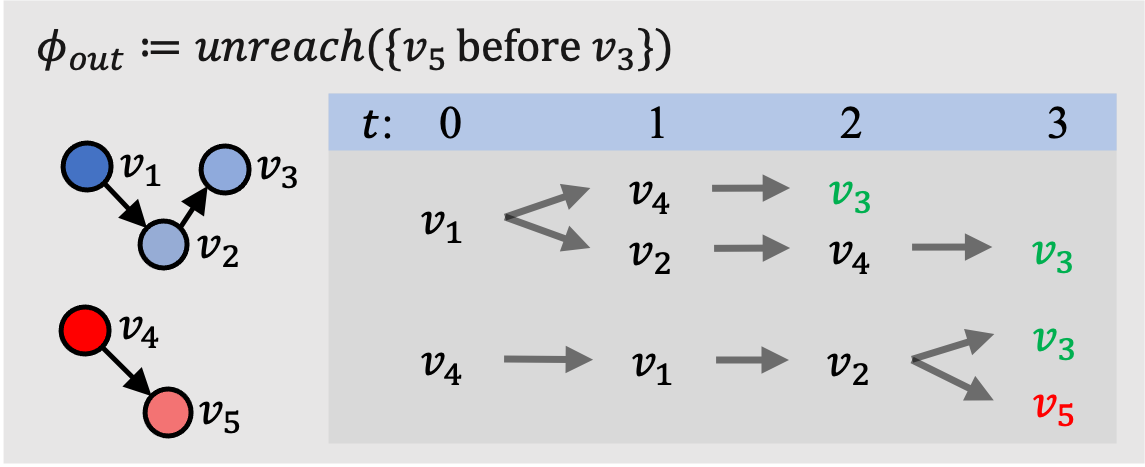}
\caption{A toy example for multi-step verification.}
\vspace{-4mm}
\label{fig:multi} 
\end{center}
\end{wrapfigure}
trace $t$, we check whether it matches the prefix of any traces in $T$. If it does not (e.g., $\node_1 \to \node_4 \to\node_3$), we do not need to continue growing that trace as the property must hold for any trace with this prefix. If the partial trace matches exactly with a trace in $T$ (e.g., $\node_4 \to \node_1\to\node_2\to\node_5$), then the post-condition is violated. Finally, if it is inconclusive, then we need to continue growing the trace.

\begin{wrapfigure}{r}{0.55\textwidth}
\begin{center}
\vspace{-0.4cm}
\includegraphics[width=0.53\textwidth]{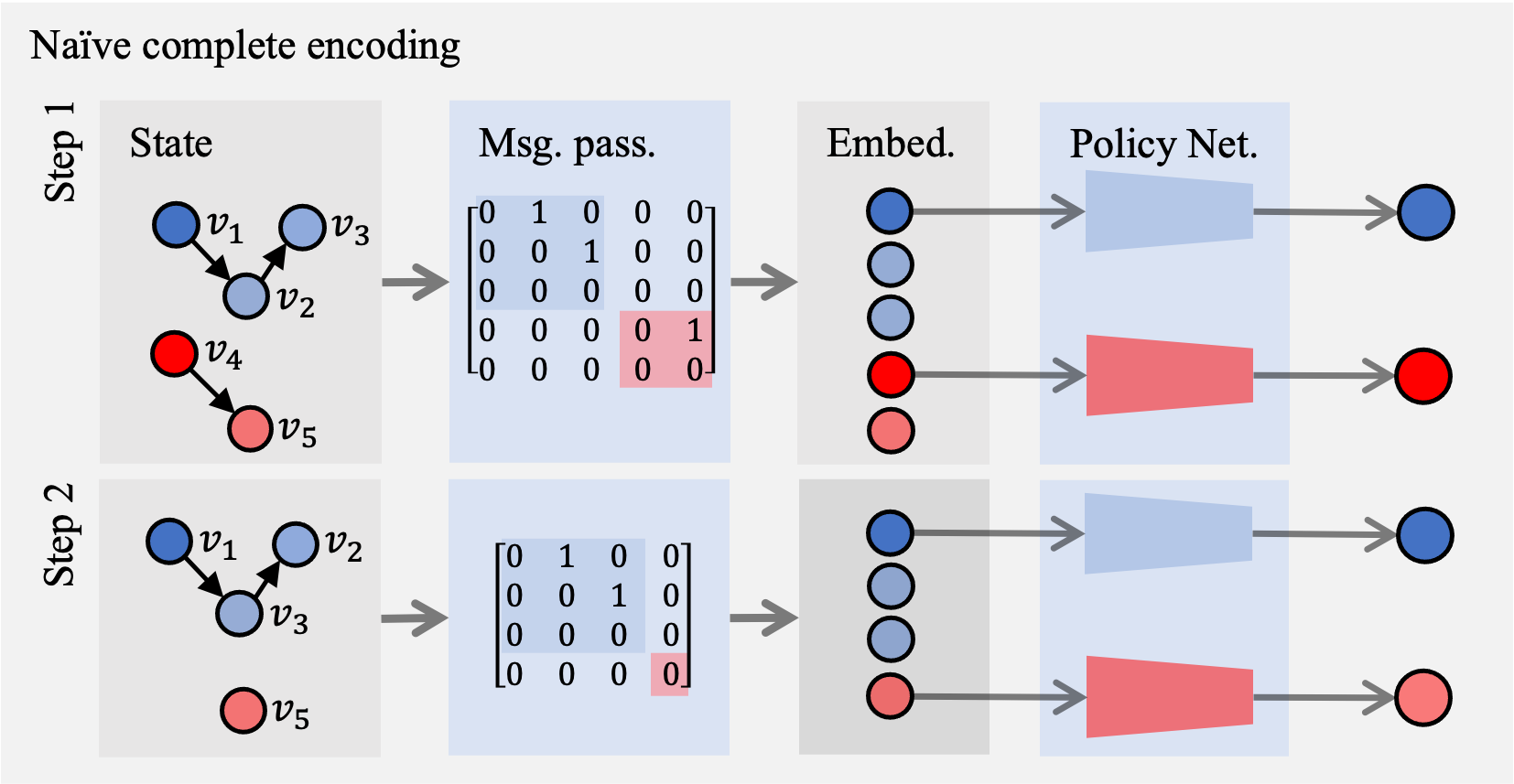}
\vspace{-0.5cm}
\caption{Example of a na\"ive encoding of two steps.}
\label{fig:multi-encoding} 
\end{center}
\vspace{-0.4cm}
\end{wrapfigure}
When computing the possible next actions from a partial trace, a complete encoding of the system requires unrolling, i.e., an encoding of the system over multiple steps. For example, as shown in Fig.~\ref{fig:multi-encoding}, suppose in the first step node $\node_4$ is scheduled; then, in the second step, the graph is updated with one removed node. A na\"ive encoding would re-encode the entire neural network for step 2 and invoke the single-step engine on this widened neural network to check whether $\out_1@2\geq \out_5@2$ and $\out_5@2\geq \out_1@2$ are respectively feasible under the additional constraint that $\out_4@1 \geq \out_1@1$. A na\"ive encoding that introduces a fresh encoding of the network for each time step quickly becomes intractable. On the other hand, an incomplete encoding that ignores the previous time steps and only encodes the current step can find spurious counter-examples. 

\begin{wrapfigure}{r}{0.55\textwidth}
\begin{center}
\vspace{-0.4cm}
\includegraphics[width=0.55\textwidth]{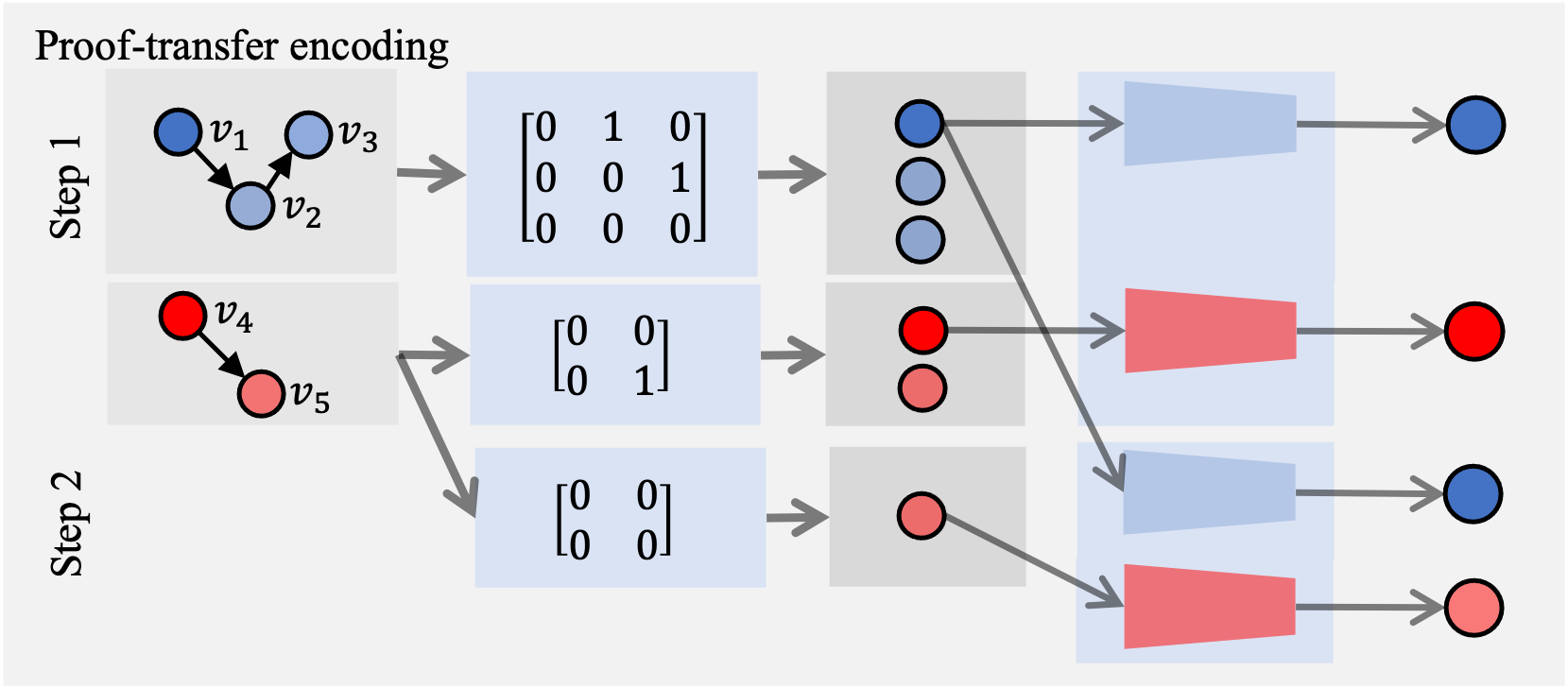}
\caption{Example of the proof-transfer encoding of the same two steps as in Fig.~\ref{fig:multi-encoding}.}
\label{fig:multi-encoding-pf} 
\end{center}
\vspace{-0.5cm}
\end{wrapfigure}
In order to obtain an efficient encoding that still tracks all constraints across time steps, we propose building a meta-network that captures the \emph{incremental} changes in the network structure across time steps and reusing the parts that do not change. We refer to this graph as a \emph{proof-transfer network}. Fig.~\ref{fig:multi-encoding-pf} shows the construction of the proof-transfer network for the same two steps as in Fig.~\ref{fig:multi-encoding}. In particular, the removal of $\node_4$ only affects nodes in the bottom job (i.e., $\node_5$) and the message passing steps for the other three nodes remain unchanged. This means that at each time step, we can reuse the GNN encoding for all but one job DAG (a disconnected component of the input graph), which results in significantly slower growth in the size of the encoding as the time step increases.

\section{Preliminaries}
\label{sec:prelim}

\subsection{Graphs and Graph Neural Networks \label{subsec:gcn}}
\begin{definition}[Graph] We define a graph $G$ with $N$ nodes as a tuple $(A, X)$ where $A$ is an $N\times N$ adjacency matrix and $X$ is the set of node attributes.
There is an edge from node $\node_i$ to $\node_j$ if $A_{ij} = 1$.
The neighborhood of a node $\node_i$ is defined as $N(\node_i) = \{\node_j |A_{ij} = 1\}$. 
The node attributes $\mathbf{X} \in \mathbb{R}^{N\times d}$ make up a node feature matrix with $x_i \in \mathbb{R}^d$ representing the feature vector of a node $v_i$. The input to a GNN is a graph $G = (A,X)$. In the case of job scheduling, an input graph can be disjoint, which is useful for modeling collections of jobs. 
\end{definition}

\paragraph{Graph Convolutional Networks.} 
Graph Neural Networks (GNNs) \cite{DuvenaudMABHAA15,NiepertAK:16,KipfW:17} are a class of deep neural networks for supervised learning on graphs. They have been successfully employed for node, edge, and graph classification in a variety of real-world applications including recommender systems \cite{YingHCEHL:18}, protein prediction \cite{ FoutBSB:17}, and malware detection \cite{WangCYLNTGLCY:19}. 
We focus on an important and widely used subset of GNNs called spatial graph convolutional networks (GCN)~\cite{wu2020comprehensive}.
Given an input graph, a GCN generates an embedding $\embed_i \in \mathbb{R}^d$ for each node $\node_i$ by aggregating its own features $\feat_i$ and all nodes reachable from $\node_i$ through a sequence of \emph{message passing} steps. In each message passing step, a node $\node$ whose neighbors have aggregated messages from all of their children  computes its own embedding as:
\[
\embed_i = g\Big[\sum_{v_j\in N(v_i)}{f(\embed_j)}\Big] + \feat_i    
\]
where $f(\cdot)$ and $g(\cdot)$ are feed-forward neural networks.%
\footnote{This definition of message passing covers common forms of GCNs as seen in \cite{khalil2017learning,KipfW:17,defferrard2016convolutional}.}
In practice, message passing is performed for a fixed number of rounds: in the first round, message-passing steps are performed on a pre-selected initial set of nodes,
and in a subsequent round, message-passing steps are performed on neighbors of the nodes that participated in the previous round.

A distinct characteristic of GCNs compared to previous GNN architectures is that there are no cyclic mutual dependencies in message passing, i.e., a GCN can be unrolled. 
However, unlike the neural networks targeted by existing verifiers, the networks obtained from unrolling are much deeper (100+ layers in practice).
Furthermore, the unrolled networks are not simple feed-forward networks but also contain complex residual connections. For simplicity, we assume that the unrolling results in a feed-forward network for the rest of this paper unless specified otherwise. Our approach for handling residual connections is described in Sec.~\ref{subsec:skip}.

\paragraph{Node regression/classification with GCN}
The embeddings $\embed_i$ after message passing can subsequently be used for different graph analytics tasks. We focus on node prediction tasks where the goal is to predict on nodes indexed by a set $\Theta$. Note that $\Theta$ might not include all the nodes in the input graph, as for a scheduler some of the nodes may not be eligible for scheduling at a given time step due to sequential dependencies.
The prediction $\out_i$ for a node $\node_i$ is computed by feeding the node embedding $\embed_i$ into a feed-forward network $\pred$. We allow the flexibility to augment the input to $\pred$ with an additional non-linear embedding $\embedMore$, computed from the node features and embeddings via a summary network $\summary$: $\embedMore= \summary(\{(\feat_i, \embed_i), \node_i\in G\})$. In short, $\out_i := \pred(\embed_i, \summary(\{(\feat_i, \embed_i), \node_i\in G\}))$.

\subsection{Verification of GNNs \label{subsec:spec}}
We consider verifying a specification $\phi$ over a GNN $F$, where $\phi$ has the form $\phi_{in} \to \phi_{out}$. 
The pre-condition $\phi_{in}$ defines a set of inputs to $F$, and the specification states that for each input point satisfying $\phi_{in}$, the post-condition $\phi_{out}$ must hold.
In this work, we limit the form of $\phi_{in}$ to describing a set of inputs where the graph structure is constant and the node features are defined by a conjunction of linear constraints. Formally, given an input graph $G = (A, X)$ and a GNN $F(\msg, \pred, \summary)$ with message-passing component $\msg$ parameterized by $g$ and $f$, prediction network $\pred$, and summary network $\summary$, the concrete behavior of $F$ can be expressed with the following set of constraints ($\feat_i$'s, $\embed_i$'s, $\embedMore$, and $\out_i$'s are interpreted as real-valued variables):
\begin{equation}\label{eq:graph_concrete}
    M := 
\begin{cases}
    \phi_{in}(\feat_1,\dots, \feat_N)\tikzmark{top}\\
    \embed_1, \ldots, \embed_N = \msg(\{\feat_i, i \in [1,N]\}, A)\tikzmark{right}\\
    \embedMore = \summary(\{\feat_i,\embed_i | \node_i \in G\})\tikzmark{bot}\\ 
    \bigwedge_{i\in\Theta} \out_i = \pred(\embed_i, \embedMore)\\
\end{cases}
\end{equation}
\AddNote{top}{bot}{right}{\scriptsize $:=M_{in}$}
We use $M_{in}$ to denote all outputs before the prediction network $\pred$ is applied. This will be used when we define the node abstraction scheme in \secref{sec:abtraction}.
The verification problem is to check whether $M\to \phi_{out}$ is valid, or equivalently, whether $M\land \neg\phi_{out}$ is unsatisfiable. Under the latter interpretation, the verification problem is to show that under the constraints $M$, the ``bad''states described by $\neg \phi_{out}$ cannot be reached.

\paragraph{Single-step post conditions.} We support single-step post conditions of the form $\phi_{simp} := \bigvee_j \psi_j(P)$, where $\psi_j(P)$ is an \emph{atomic} linear constraint over the output variables
, i.e., $(\sum_{\ell\in[1,N]} a_\ell \cdot \out_\ell) \bowtie b$, where $a_\ell$ and $b$ are constants and $\bowtie \in \{=, <, \leq\}$. We refer to this as a \emph{simple} post condition because the ``bad'' states can be described as a conjunction of linear constraints, and checking $\phi_{simp}$ amounts to showing that $M\land \bigwedge_j \neg\psi_j(P)$. Moreover, we support richer post conditions that state that multiple simple post conditions must hold simultaneously: $\phi_{complex} = \bigwedge_i \phi^i_{simp}$. 
We describe our novel abstraction to efficiently handle such post conditions in \secref{sec:abtraction}.

\subsection{GNN-based job scheduling \label{prelim:gnn-verify}}

As a proof of concept, we focus on verifying the state-of-the-art GNN-based job scheduler, Decima~\cite{decima}.
Formally, the input to the scheduler is a graph $G$ with $K$ disconnected components $G_1, \ldots, G_K$, representing the current state of the cluster. Each disconnected component $G_k$ is a job DAG. The output of Decima is a single score $\out_i$ for each frontier node $\node_i$. We use $\frontier(G)$ to denote the frontier nodes. The node to schedule next is $\argmax_{\node_i\in\frontier(G)}\out_i$. 
Our goal is to reason about Decima not only in the single-step setting, but also in the multi-step setting. which we describe next.

\paragraph{Multi-step setting.} We consider a setting with an initial state $G=(A, X)$, a GNN-based scheduler $F$, and a cluster environment $\trans(G, \node)\mapsto G'$, which takes the current state $G$ and a node $\node\in\frontier(G)$, and outputs a new state $G'$ representing the new state after $\node$ is scheduled. 
We restrict $\mathcal{T}$ to perform two types of graph updates: 1) removal/addition of nodes; and 2) affine transformations of features of existing nodes: $x_i' \mapsto A \vect{x_i} + \vect{b}$. This precisely captures the actual cluster environment for a wide range of initial states. For simplicity, we only consider static job profiles with no new incoming jobs (i.e., only node removal and no node addition). 
Given $G$, $F$, and $\trans$, we refer to a finite sequence of scheduling decisions $F(G), F(\trans(G, F(G))), \ldots$ as a \emph{trace}.

\paragraph{Multi-step post conditions.}
We consider multi-step post conditions defined in terms of trace reachability. 
The post condition $\phi_{out}$ is of the form $\unreach(T)$ where $T$ is a finite set of finite traces of possibly different lengths.
The specification states that none of the traces in $T$ are feasible if the initial state satisfies $\phi_{in}$. As we will see in \secref{sec:spec}, this form of specifications allows us to specify meaningful multi-step properties for the job-scheduler.

\section{Forward-backward analysis}
\label{sec:backward}

 \begin{wrapfigure}{r}{0.63\textwidth}
\begin{minipage}[t]{0.63\textwidth}
\vspace{-0.7cm}
\begin{algorithm}[H]
  \small
  \begin{algorithmic}[1]
    \State {\bfseries Input: neural network $f$ and specification $\phi_{in}\to \phi_{out}$} 
    \State {\bfseries Output: $\hold/\nothold/\unknown$}
    \Function{forwardBackwardAnalysis}{$\phi$}
    \State $\asFor \mapsto \top,\asBack \mapsto \top$
    \While {$\neg\converged()$}
    \State $\asFor \mapsto \forward(f,\{a^\hiddenlayer\}, \{a'^\hiddenlayer\}, \phi_{in})$ \label{lst:line:forward}
    \If{$\cond(a^{\numlayers}, \neg\phi_{out}) = \bot$} \label{lst:line:if_forward}
    \State {\bf return} \hold
    \EndIf
    \State {$bounds\mapsto []$}
    \For {$\hiddenlayer = K, K-1, \ldots, 2$} \label{lst:line:backward}
    \State {$bounds[\hiddenlayer] \mapsto \computeBounds(f, \hiddenlayer,\{a^\hiddenlayer\},\{a'^\hiddenlayer\}, \phi_{out})$}
    \State $\aBack^{\hiddenlayer} \mapsto \cond(\aFor^\hiddenlayer,  (x^{(\hiddenlayer)}\in bounds[\hiddenlayer])) \sqcap \aBack^\hiddenlayer$
    \If{$\aBack^{\hiddenlayer}= \bot$} \label{lst:line:if_backward}
    \State {\bf return} \hold
    \EndIf
    \EndFor
    \EndWhile
    \State {\bf return} $\checkSat(\phi_{in} \land (\underaccent{\hiddenlayer\in [1,\numlayers]}{\bigwedge} (\varphi^{(\hiddenlayer)}_{\text{non-linear}} \land \varphi^{(\hiddenlayer)}_{\text{linear}}))\land \neg\phi_{out})$
    \EndFunction
  \end{algorithmic}
  \caption{Forward-Backward Analysis.\label{alg:forbackward}}
\end{algorithm}
\vspace{-0.8cm}
\end{minipage}
\end{wrapfigure}
 
In this section, we describe our forward-backward abstraction refinement framework in more formal terms. We consider an $\numlayers$-layer feed-forward neural network $f\colon\mathbb{R}^{n_0} \to \mathbb{R}^{n_\numlayers}$, where $n_0$ and $n_\numlayers$ are the number of input and output neurons respectively. For an input $x$, we use $f_{\hiddenlayer}(x)$ and $f_{\hiddenlayer:\numlayers}(x)$ respectively to denote the network output at an intermediate layer $\hiddenlayer$ and all layers between $l$ and $\numlayers$. We consider the affine and the non-linear activation layers as separate. As illustrated in \secref{sec:overview}, our key insight is to iteratively refine the abstraction obtained from a forward abstract interpretation with an LP-based backward pass and vice-versa. \algoref{alg:forbackward} shows the pseudocode for our framework. Next we describe each step in greater detail and prove soundness properties.

\subsection{Forward abstract interpretation} The forward analysis in our framework is generic as it can be instantiated with any sub-polyhedral domain including the popular domains for neural network verification such as Boxes \cite{reluval}, Zonotope \cite{AI2,deepz}, DeepPoly \cite{deeppoly}, or Polyhedra \cite{singh2017fast}. 
We use $\sA_n$ to denote an abstract element overapproximating the concrete values of $n$ numerical variables. We require that the abstract domain is equipped with the following components:
\begin{itemize}
\item A  concretization function $\gamma_n\colon \sA_n\to\mathcal{P}(\mathbb{R}^n)$ that computes the set of concrete points from $\mathbb{R}^n$ represented by an abstract element $a\in \sA_n$.
\item A bottom element $\bot\in\sA_n$ such that $\gamma_n(\bot)=\emptyset$.

\item A sound abstraction function $\alpha_n\colon\mathcal{P}(\mathbb{R}^n)\to\sA_n$ that computes an abstract element overapproximating a region $\phi_{in}\in\mathcal{P}(\mathbb{R}^n)$ provided as input to  the neural network. We have $\phi_{in}\subseteq\gamma_n(\alpha_n(\phi_{in}))$ for all $\phi_{in}\in\mathcal{P}(\mathbb{R}^n)$. Note that we do not require $\alpha_n$ to compute the smallest abstraction for $\phi_{in}$, however, the input regions considered in our experiments can be exactly abstracted with common domains such as DeepPoly and Zonotope.
\item A bounding box function $\iota_n\colon \sA_n\to \sR^n\times\sR^n$, where $\gamma_n(a)\subseteq\prod_i[c_i,d_i]$ for $(\vc,\vd)=\iota_n(a)$ for all $a\in\sA_n$.

\item A sound conditional transformer $\cond(a, C)$ that for each $a\in\sA_n$ and set of linear constraints $C$ defined over $n$ real variables satisfies $\gamma_n(\cond (a,C))\subseteq \gamma_n(a)$, i.e., the conditional output does not contain more points than the input $a$. 
\item A sound abstract transformer $\forwardT \colon\sA_m\to\sA_n$ for each layerwise operation $\operation \colon\mathbb{R}_m\to\mathbb{R}_n$ (e.g., affine, non-linear activations, etc.) in the neural network.

\item A sound meet transformer $\sqcap$ for each $a,a'\in \sA_n$ satisfying $\gamma_n(a\sqcap a') \subseteq \gamma_n(a)$ and $\gamma_n(a\sqcap a') \subseteq \gamma_n(a')$.
\end{itemize}
 
 Our framework associates an abstract element $a^\hiddenlayer$ from the forward pass and an element $a'^\hiddenlayer$ from the backward pass with each layer $\hiddenlayer$. Both elements are constructed such that they individually overapproximate the set of concrete values at layer $\hiddenlayer$ with respect to $\phi_{in} \land \neg\phi_{out}$ at each iteration of the while loop in \algoref{alg:forbackward}. Initially, all elements are $\top$. 

\paragraph{Constructing forward transformers.}
The forward pass shown at Line~\ref{lst:line:forward} in \algoref{alg:forbackward} first constructs an abstraction of the input region $a^1=\alpha_{n_{0}}(\phi_{in})$. We propagate $a^1$ through the different layers of the network via a novel construction that creates new \emph{higher order} abstract transformers $\forwardbackwardT$ from existing $\forwardT$ for each operation $\operation$. For a layer $\hiddenlayer$, the construction of $\forwardbackwardT$ takes $\forwardT$ and the abstract elements $a^{\hiddenlayer-1}, a'^{\hiddenlayer-1}$ at layer $\hiddenlayer-1$ as inputs. Its output is the new forward abstract element at layer $\hiddenlayer$:
\begin{equation}\label{eq:forward_T}
    a^\hiddenlayer=\forwardbackwardT(\forwardT,a^{\hiddenlayer-1},a'^{\hiddenlayer-1})=\forwardT(a^{\hiddenlayer-1}) \sqcap \forwardT(a'^{\hiddenlayer-1})
\end{equation}
We next prove the soundness of our construction.

\begin{theorem}\label{thm:forwardback_theorem}
$\forwardbackwardT$ is a sound abstract transformer, that is, given an input that includes all concrete values at layer $l-1$ with respect to $\phi_{in} \land \neg\phi_{out}$, the transformer's output includes all concrete values possible at layer $l$ with respect to $\phi_{in} \land \neg\phi_{out}$.
\end{theorem}
\begin{proof}
\sloppy
Let $S^{\hiddenlayer-1}$ and $S^{\hiddenlayer}$ respectively be the set of concrete values at layer $\hiddenlayer-1$ and $\hiddenlayer$ with respect to $\phi_{in}\land \neg\phi_{out}$. For sound abstractions $a^{\hiddenlayer-1}, a'^{\hiddenlayer-1}$ , we have $S^{\hiddenlayer-1} \subseteq \gamma_n(a^{\hiddenlayer-1})$ and $S^{\hiddenlayer-1} \subseteq \gamma_n(a'^{\hiddenlayer-1})$. Since $\forwardT$ is sound, $S^{\hiddenlayer} \subseteq \gamma_n(\forwardT(a^{\hiddenlayer-1}))$ and $S^{\hiddenlayer} \subseteq \gamma_n(\forwardT(a'^{\hiddenlayer-1}))$. Thus  $S^{\hiddenlayer}\subseteq \gamma_n(\forwardT(a^{\hiddenlayer-1})) \cap \gamma_n(\forwardT(a'^{\hiddenlayer-1})) \subseteq \gamma_n(\forwardT(a^{\hiddenlayer-1}) \sqcap \forwardT(a'^{\hiddenlayer-1}))=\gamma_n(\forwardbackwardT(\forwardT,a^{\hiddenlayer-1},a'^{\hiddenlayer-1}))$
\end{proof}

\begin{corollary}
The output of $\forwardbackwardT$ is included in the output of $\forwardT$ for inputs soundly abstracting the concrete values with respect to $\phi_{in} \land \neg\phi_{out}$. That is, the output of $\forwardbackwardT$ is at least as precise as $\forwardT$.
\end{corollary}



\eqref{eq:forward_T} invokes the original transformer $\forwardT$ twice along with the $\sqcap$ transformer. For most popular domains, $\sqcap$ is asymptotically cheaper than $\forward$, therefore the asymptotic cost of $\forwardbackwardT$ is same as $\forwardT$.
Since $\forwardbackwardT$ obtained for each layer is sound, the forward propagation produces  $\gamma_{n_\numlayers}(a^\numlayers) \supseteq f(\phi_{in}) \supseteq f(\phi_{in}) \land \neg\phi_{out} $ at the output layer. In this section, we assume that the set of bad outputs $\neg\phi_{out}$ can be described as a conjunction of linear constraints. In the next section, we will consider a more general class of $\neg\phi_{out}$ containing disjunctions also. For our restriction here, we can compute $\cond(a^\numlayers, \neg\phi_{out})$, and if it is equal to $\bot$, we have proved that the specification must hold since $\gamma_{n_\numlayers}(\cond(a^\numlayers,\neg\phi_{out})) \supseteq f(\phi_{in}) \land \neg\phi_{out}$ (Line~\ref{lst:line:if_forward} in \algoref{alg:forbackward}).

\subsection{Backward abstract interpretation} The forward propagation ignores the post condition $\phi_{out}$ during the construction of the abstract element $a^\numlayers$. A refinement of the abstraction can be obtained by taking $\phi_{out}$ into consideration. The backward pass  shown at Line~\ref{lst:line:backward} in \algoref{alg:forbackward} is designed to accomplish this refinement.
%

The backward pass first updates the backward element at the output layer $a'_{\numlayers} = \cond(a_{\numlayers},\neg\phi_{out})\sqcap a'_{\numlayers}$. It can be seen that this update is sound.
The backward pass at a non-output layer $k$ performs two steps. First, we use Linear Programming (LP) to compute refined lower and upper bounds $\vect{c^{\hiddenlayer}},\vect{d^{\hiddenlayer}}$, taking into consideration both $\phi_{out}$ and a linear overapproximation of the network behavior from layer $\hiddenlayer$ to the output layer with respect to $\phi_{in} \land \neg \phi_{out}$. 
We then refine the backward abstract element using the domain conditional transformer $\cond$ and the linear constraints $c^{\hiddenlayer}_i\le x^{\hiddenlayer}_i\le d^{\hiddenlayer}_i$ for each neuron $x_i^\hiddenlayer$ in layer $\hiddenlayer>1$, i.e., we compute

\begin{equation}\label{eq:backward}
a'^\hiddenlayer=\cond(a^\hiddenlayer,(\bigwedge_i c_i\le x_i^{\hiddenlayer}\le d_i)) \sqcap a'^\hiddenlayer.
\end{equation}

We now describe our linear encoding of the network for computing the refined interval bounds $\vect{c^{\hiddenlayer}},\vect{d^{\hiddenlayer}}$ of the neurons in layer $\hiddenlayer$. We start with an empty set of constraints and iteratively collect linear constraints $\varphi^{k}(\vect{x^{k-1}},\vect{x^{k}})$ overapproximating the network behavior with respect to $\phi_{in} \land \neg \phi_{out}$ for each layer $k>\hiddenlayer$. If $k$ is an affine layer, then we add 
constraints of the form $x^{k} =  A^{k-1} \vect{x}^{k-1} + b^{k-1}$, where $A^{k-1},b^{k-1} \in \mathbb{R} $ are the learned weights and biases respectively of the affine layer. If $k$ is an activation layer, then we add the constraints from a linear approximation~\cite{deeppoly} of the layerwise activation $\vect{x}^{k}=\sigma(\vect{x}^{k-1})$. The linear approximation can be obtained directly using the corresponding domain transformer $\forwardT$ or a more precise relaxation based on the constraints from the forward abstract element. If $k$ is the output layer, then we add the constraint $\neg \phi_{out}$ to our encoding. For each $k$, we also add the constraints $C^{k}$ obtained from the bounding box $\iota_n$ of the abstract element $a^{k}$ at layer $k$.

Let $\varphi^{\hiddenlayer}$ denote the conjunction of constraints collected above. To obtain the refined lower and upper bounds for layer $\hiddenlayer$ with $m$ neurons, we need to solve the following two LPs for each neuron $x_i^{\hiddenlayer}$:
\begin{equation}\label{eq:lps}
c^{\hiddenlayer}_i = \min_{\substack{\vect{x^{\hiddenlayer}},\ldots,\vect{x^{\numlayers}}\\
\text{s.t. }\varphi^{\hiddenlayer}}} x^{\hiddenlayer}_i,
d^{\hiddenlayer}_i = \max_{\substack{\vect{x^{\hiddenlayer}},\ldots,\vect{x^{\numlayers}}\\
\text{s.t. }\varphi^{\hiddenlayer}}} x^{\hiddenlayer}_i\\
\end{equation}

While we refine all neurons in the backward step to gain maximum precision, it is possible to tune the cost and the precision of the backward pass by selectively refining only a subset of the neurons. 
Note that we do not refine the backward element at the input layer, and it is always $\top$ which preserves soundness.

\begin{theorem}\label{thm:LP_refinement}
For each layer $\hiddenlayer$, the refined bounds $\vect{c^{\hiddenlayer}},\vect{d^{\hiddenlayer}}$ computed by the LP overapproximate the network output with respect to $\phi_{in} \land \neg\phi_{out}$.
\end{theorem}
\begin{proof}
We show that the constraints in the LP overapproximate the network behavior from layer $\hiddenlayer$ to the output which guarantees the soundness of the refined bounds. The proof is by induction. For the base case, the LP constraints at the output layer satisfy $\neg\phi_{out} \land C^\numlayers \supseteq f(\phi_{in}) \land \neg \phi_{out}$. For the inductive step, suppose $\underaccent{k\in [\hiddenlayer + u+ 1,\numlayers]}{\bigwedge} (\varphi^{k}(\vect{x^{k-1}},\vect{x^{k}}) \land C^{k}) \supseteq f_{\hiddenlayer+u+1:\numlayers}(\phi_{in}) \land \neg \phi_{out}$ holds with $u \geq 0$. For the layer $\hiddenlayer+u$, the LP will add the conjunction satisfying  $\varphi^{k}(\vect{x^{\hiddenlayer+u-1}},\vect{x^{(\hiddenlayer+u)}}) \land C^{\hiddenlayer+u}\supseteq f_{\hiddenlayer+u:\hiddenlayer+u+1}(\phi_{in}) \land \neg \phi_{out}$ (since both the box constraints from $C^{\hiddenlayer+u}$ and the constraints from $\varphi^{k}(\vect{x^{(\hiddenlayer+u-1)}},\vect{x^{(\hiddenlayer+u)}})$ overapproximate $f_{\hiddenlayer+u:\hiddenlayer+u+1}(\phi_{in})$). Therefore, $ \underaccent{k\in [\hiddenlayer + u,\numlayers]}{\bigwedge} (\varphi^{k}(\vect{x^{k-1}},\vect{x^{k}}) \land C^{k}) \supseteq f_{\hiddenlayer+u:\numlayers}(\phi_{in}) \land \neg \phi_{out}$ and the induction holds.  
\end{proof}
Since the bounds computed by the LP are sound, the update \eqref{eq:backward} is sound at each intermediate layer $\hiddenlayer$. Therefore the backward pass computes a sound approximation at each iteration of the while loop of \algoref{alg:forbackward}. Due to the soundness of the backward pass, if we find that a backward element at a layer $\hiddenlayer>1$ became $\bot$ after applying \eqref{eq:backward}, then we can soundly return \hold\  
(Line~\ref{lst:line:if_backward} of \algoref{alg:forbackward}).

\begin{theorem}\label{thm:backward_refine}
For a layer $\hiddenlayer$, let $a'^{\hiddenlayer}$ and $a'^{\hiddenlayer}_{\text{new}}$ be the backward abstract elements before and after the refinement respectively. Then, $\gamma_n(a'^{\hiddenlayer}_{\text{new}}) \subseteq \gamma_n(a'^{\hiddenlayer})$ holds, i.e., the backward pass does not make the backward abstraction less precise. 
\end{theorem}
\begin{proof}
Follows from the definition of $\cond$ and $\sqcap$.
\end{proof}

\begin{corollary}
Let $q$ be the total number of neurons in the network $f$, then the complexity of computing \eqref{eq:lps} for all neurons in layers $1<k\leq K$ is $O(q \cdot LP(p,q) )$ where $LP(p,q)$ is the cost of solving an LP with $p$ constraints defined over $q$ variables.
\end{corollary}

\subsection{Iterative forward-backward refinement} 

We perform an iterative refinement procedure in \algoref{alg:forbackward} by repeatedly performing forward and backward analysis till a stopping criteria is met. Each forward-backward pass results in elements at least as precise as in previous iterations due to Theorem~\ref{thm:backward_refine} and \eqref{eq:forward_T}.
 For arbitrary choices of abstract elements and transformers allowed by our framework, the forward-backward analysis is not guaranteed to converge to a fixed point: another iteration of the while loop does not refine the analysis results. In practice, for piecewise-linear activations like Leaky ReLU used in Decima, we terminate the analysis when a refinement round does not fix the phase of any activations. This is usually achieved within 6 iterations.

If, after the forward-backward analysis has finished, no abstract element has been refined to $\bot$, we resort to a non-linear encoding of the network outputs $\underaccent{\hiddenlayer \in [1,\numlayers]}{\bigwedge} \varphi^{\hiddenlayer}_{\text{non-linear}}$ combined with linear constraints $\underaccent{\hiddenlayer \in [1,\numlayers]}{\bigwedge} \varphi^{\hiddenlayer}_{\text{linear}}$ from our final abstract elements. For example, we use a MILP encoding for piecewise linear activations such as Leaky ReLU used in Decima which results in complete verification: the property is satisfied if and only if the set of constraints $\phi_{in} \land (\underaccent{\hiddenlayer\in [1,\numlayers]}{\bigwedge} (\varphi^{\hiddenlayer}_{\text{non-linear}} \land \varphi^{\hiddenlayer}_{\text{linear}}))\land \neg\phi_{out}$ is unsatisfiable. The non-linear solver benefits in speed from a reduction in the search area/branches due to the precise constraints added from our refined abstraction.

\begin{theorem}
\algoref{alg:forbackward} is sound, i.e., $\phi_{in} \land \neg \phi_{out}$ is unsatisfiable when the algorithm returns \hold.
\end{theorem}
\begin{proof}
The individual steps in \algoref{alg:forbackward} are sound.
\end{proof}

\subsection{Handling GNN architectures.\label{subsec:skip}}

GNNs often contain residual connections due to message passing. In those cases, we still need to make sure that when performing abstraction refinement for a certain layer $k$, we have already refined all subsequent layers that take $k$'s outputs as inputs. To obtain this refinement order, we first construct a DAG from the neural network architecture where each node represents a layer and an edge exists from layer $k$ to layer $\ell$ if the output of the former feeds into the latter. Drawing inspirations from data-flow analysis~\cite{aho2007compilers}, the backward abstraction refinement is conducted in post-order.

\section{Node abstraction for GNN verification.}
\label{sec:abtraction}

\algoref{alg:forbackward} is our core algorithm for verifying a single step property $\phi_{in}\to\phi_{out}$ where $\neg\phi_{out}$ is of the form $\bigwedge \Sigma a_i \cdot \out_i \bowtie c_i$ (i.e., a conjunction of linear constraints over the output layer) for a feed-forward neural network. However, only allowing this form of bad outputs is restrictive as in practice the bad outputs specified in many properties (e.g., robustness, strategy-proofness) are disjunctive sets. In particular, the output property often specifies that a subset of output neurons $\Theta$ all satisfy certain simple post-condition $\phi$, i.e., 
\begin{equation}
\hspace{-0.8cm}\phi_{out} := \bigwedge_{{\out_i}\in{\Theta}} \phi({\out_i}) \label{eq:dnf}    
\end{equation}
Existing works often handle this type of post condition by considering each disjunct individually: we can use \algoref{alg:forbackward} (or any procedure that can handle simple post-conditions) to check whether $\phi_{in} \to\phi(\out_i)$ holds for each output variable $\out_i $. The original property holds if the solver returns $\hold$ every time, and is violated if the solver returns $\nothold$ for one of the disjuncts. However, this strategy could be inefficient when the number of disjuncts is large. 
Instead, in this section, we describe a general procedure to efficiently handle post conditions of this form tailored for node prediction/classification tasks in GNNs. We believe our contributions for verification presented here can be adapted to handle GNN application domains other than job-scheduling such as recommender systems and malware detection.

Our key insight is that each score $s_i$ is computed by applying the same transformation to the node embedding $\embed_i$, $\vect{\out_i} := \pred(\embed_i, \summary(\{\feat_i,\embed_i | \node_i \in G\}))$. Therefore, to simultaneously reason about the GNN's output for a group of nodes $\{\node_i | i\in\Theta\}$, it is natural to consider an abstraction that treats $\{\embed_i, i\in \Theta\}$ as an equivalence class.
Given a set of constraints $M$ as defined in \eqref{eq:graph_concrete} that exactly captures the concrete behaviors of the GNN with respect to $\phi_{in}$, we can construct an abstraction $\rlx{M}$ by mapping each constraint ``$\out_i = \pred(\embed_i, \embedMore)$'' to ``$\abst{\out} = \pred(\abst{\embed}, \embedMore) \land I(\abst{\embed})$'', where $\abst{\embed}$ and $\abst{\out}$ are free real-valued variables and $I$ is an \emph{invariant} which we will describe next:
\begin{align*}
    \rlx{M} := 
\begin{cases}
    M_{in} \text{ (defined in \eqref{eq:graph_concrete})}\\
    \abst{\out} = \pred(\abst{\embed}, \embedMore)\land I(\abst{\embed})
\end{cases}
\end{align*}

\begin{definition}[Soundness]
We say $\rlx{M}$ is a \emph{sound} abstraction of $M$ with respect to a simple output property $\phi$, if 
\[
\rlx{M}\to \phi(\abst{\out}) \implies M\to \bigwedge_{i\in \Theta} \phi(\out_i).
\]
\end{definition}
By definition, given a sound abstraction, we can prove the original property by proving a simple specification on $\rlx{M}$.

\begin{lemma}
if $M_{in} \to I(\embed_i)$ for each $i \in \Theta$, then $\rlx{M}$ instantiated with $I$ is sound.\label{lemma:sound}
\end{lemma}
\begin{proof}
Suppose $\rlx{M}\to \phi(\abst{\out})$. It follows that $\forall \embed, M_{in} \land I(\embed) \to \phi(\pred(\embed,\embedMore))$. For each $i\in \Theta$, we can instantiate $\embed$ with $\embed_i$ and get $M_{in} \land I(\embed_i) \to \phi(\pred(\embed_i,\embedMore))$. Since $M_{in} \to I(\embed_i)$, we have $M_{in} \to \phi(\pred(\embed_i,\embedMore))$. This is equivalent to $M_{in} \land (\out_i= \pred(\embed_i,\embedMore)) \to \phi(\out_i)$. Notice that $M_{in} \land (\out_i= \pred(\embed_i,\embedMore))$ is a subset of constraints in $M$. Therefore, $M \to \phi(\out_i)$.
\end{proof}

\begin{wrapfigure}{r}{0.6\textwidth}
\begin{minipage}[t]{0.6\textwidth}
\vspace{-0.7cm}
\begin{algorithm}[H]
  \small
  \begin{algorithmic}[1]
    \State {\bfseries Input: a message passing component $m$, a summary component $\summary$ and a prediction component $\pred$, and a specification $\phi : \phi_{in}\to \bigwedge_{i\in{\Theta}} \phi(\out_i)$} 
    \State {\bfseries Output: $\hold/\nothold/\unknown$}
    \Function{checkWithNodeAbtraction}{$m, \summary, \pred, \phi$}
    \State $A, C\mapsto \Theta, \emptyset$
    \While {$C \neq \Theta$}
    \State $a_{e_1}, \ldots, a_{e_N} \mapsto \forward(M, \phi_{in})$
    \State $I \mapsto \convexRelaxation(a_{e_1}, \ldots, a_{e_N})$
    \State $\rlx{M} \mapsto \createAbstraction(\phi_{in}, m, \summary, \pred, I)$
    \State $r \mapsto\forwardBackwardAnalysis(\rlx{M}, \phi(\abst{\out}))$
    \If{$r = \hold$}
    \State $C \mapsto C \cup A$
    \State $A \mapsto \emptyset$
    \Else
    \State $a_{e_k} \mapsto \pickNode(A)$
    \State $A \mapsto A \setminus \{k\}$
    \State $t \mapsto\forwardBackwardAnalysis(F, \phi_{in}\to \phi(\out_k))$
        \If {$t = \hold$}
        \State $C \mapsto C\cup \{k\}$
        \Else 
        \State {\bf return} t
        \EndIf
    \EndIf
    \EndWhile
    \State {\bf return} \hold
    \EndFunction
  \end{algorithmic}
  \caption{Node abstraction with iterative refinement.\label{alg:node-abstract}}
\end{algorithm}
\vspace{-0.8cm}
\end{minipage}
\end{wrapfigure}
\paragraph{Node abstractions.}
While the choice of $I$ is flexible as long as it yields a sound abstraction, there is a trade-off controlled by $I$ between the difficulty of proving $\rlx{M}\to \phi(\abst{\out})$ and the likelihood that this property holds. For example, the weakest invariant is $\top$, which leaves $\embed'$ unconstrained, making it less likely that $\phi(\abst{p})$ holds. On the other hand, the strongest invariant is $\bigvee_{i\in\Theta} \abst{\embed} = \embed_i$, which just moves the disjunction in the output property to be over $\abst{\embed}$ and does not reduce the computational complexity.

A general recipe for constructing $I$ relies on the forward abstract interpretation. After performing the forward analysis on GNN $F$ with pre-condition $\phi_{in}$ using the abstract transformers $T^\#_{g}$, for each $\embed_i$, we can obtain a (typically) convex region $\psi_i$ from the abstract element $a_{e_i}$ such that $\psi_i$ over-approximates the values $e_i$ can take under the pre-condition $\phi_{in}$, or in other words, $M_{in} \to \psi_i(e_i)$. Next, we can compute a convex relaxation of the union of the convex regions and use that as our invariant $I$, i.e., $I = \conv(\bigcup_{i\in\Theta} \psi_i)$. In practice, we take $\psi_i$ to be the tightest intervals of $\embed_i$, and let $I$ be the join of those intervals. 

\begin{theorem}\label{thm:abs-sound}
If $M_{in} \to \psi_i(\embed_i)$ for each $i\in\Theta$, then the abstraction $\rlx{M}$ instantiated by $I = \conv(\bigcup_{i\in\Theta} \psi_i)$ is sound.
\end{theorem}
\begin{proof}
For any $\embed_i$ where $i\in\Theta$, by the definition of convex approximation $\psi_i(\embed_i) \to I(\embed_i)$. Therefore $M_{in} \to I(\embed_i)$ and by Lemma~\ref{lemma:sound} the statement holds.
\end{proof}

Note that $\rlx{M}$ can be viewed as the concrete semantics of a GNN $\rlx{F}$ with an augmented input $\abst{\embed}$ and one output $\abst{\out}$. Under this view, checking $\rlx{M}\to\phi(\abst{\out})$ is equivalent to checking the specification $\phi_{in}(\feat_1, \ldots, \feat_N) \land I(\abst{\embed}) \to \phi(\abst{\out})$ on $\rlx{F}$. 

\paragraph{Iterative refinement of node abstraction.}
If $\rlx{M}\to\phi(\abst{\out})$ is proved, then the original property also holds. Otherwise we cannot conclude that the original property is violated. In that case, we can obtain a refinement of the over-approximation by considering fewer nodes in the abstraction. This yields an iterative refinement procedure as described in \algoref{alg:node-abstract}. We maintain two sets of indices: $A$ is the set of node indices treated as equivalent in the abstraction and a node index $i\in C$ if $\phi_{in}\to\phi(\out_i)$ has been proved.
Given a specification of the form described in Eq.~\eqref{eq:dnf}, we start by creating an abstract network $\abst{F}$ that treats all nodes in the disjuncts as equivalent. If the simple property $\phi_{in}\to\phi_{\abst{\out}}$ can be proved on F', then we add all node indices in the equivalence class $C$. If we fail to prove the property, we refine the abstraction by considering fewer nodes in the abstraction. In particular, we heuristically pick one node (Line 13) to remove from $A$. In practice, we pick the node whose removal results in the largest decrease in the volume of the convex relaxation $I$. In the next iteration, we obtain a different abstract network $F'$ that abstracts over fewer nodes.

\begin{theorem}\label{thm:sound_complete_node_abstraction}
\sloppy
If $\rlx{M}$ is a sound abstraction of $M$ and $\forwardBackwardAnalysis$ is sound and complete, then \algoref{alg:node-abstract} is sound and complete.
\end{theorem}
\begin{proof}
If the specification can be violated, then $r$ at Line 9 must always equal $\nothold$ and each iteration picks and checks a different conjunct (Lines 14-16). Since $\forwardBackwardAnalysis$ is sound and complete, it must return $\nothold$ for one conjunct. Now suppose the specification holds. Since the relaxation is sound and $\forwardBackwardAnalysis$ is sound and complete, $C$ must only contain indices corresponding to conjuncts that hold. Moreover, each iteration must increase the size of $C$ (Lines 11 and 18). Since $\Theta$ is finite, the algorithm will return $\hold$ in finite number of steps. 
\end{proof}

\section{Reasoning over traces}
\label{sec:multi}

So far we have introduced a verification engine in \algoref{alg:node-abstract} for single-step properties with pre-condition $\phi_{in}$ that can be expressed as a conjunction of linear constraints on the input node features and post condition $\phi_{out}$ of the form described in \eqref{eq:dnf}. 
Building upon this engine, we now develop an analysis to reason about the system behavior across multiple time steps. In this setting, we want to prove that bad traces from a set $T$ are not feasible starting from any state satisfying $\phi_{in}$ that can be expressed as a conjunction of linear constraints on the input node features. 
Next, we develop a baseline algorithm that iterates over traces and performs early punning of safe traces. Then, we discuss the precision and performance trade-off in this procedure and describe an efficient encoding of the system across multiple steps that still preserves completeness.

\paragraph{Multi-step verification with trace enumeration.}
We first present a multi-step verification procedure in  \algoref{alg:multi} and then describe our optimizations for improving its speed. 
At a high-level, the algorithm searches for an initial state in $\phi_{in}$ that would result in a trace $t\in T$ by trying to compute all possible traces starting from $\phi_{in}$. The progress of the search is tracked by a stack (initialized with an emptry trace $\textsc{nil}$) containing the set of (partial) traces that need to be explored.

During the search, we pick a partial trace $t$ (Line~\ref{lst:line:trace_pick}) from the stack and check whether it matches any traces in $T$ (Line~\ref{lst:line:trace_match}). There is a match between two traces if they have the same length and the same action sequence. There are three outcomes of our check. If there is a match ($\match$), then the search terminates as we have found a (potential) violation of the property, and we return either $\nothold$ or $\unknown$ depending on whether the encoding and the base solver ($\checkDisjuncts$) is complete.  Otherwise, if no trace in $T$ has $t$ as a prefix, then the check returns $\notmatch$. In this case, we can conclude any trace with prefix $t$ is not in $T$ and move on to analyze another trace. If this check is inconclusive (i.e., no trace in $T$ is equal to $t$, but some traces have $t$ as prefix), then we must expand this trace to determine whether there is a potential property violation. That is, we need to compute the possible next actions of the GNN agent conditioned on the initial state $\phi_{in}$, the transition system $\mathcal{T}$, and the current trace $t$ (Line~\ref{lst:line:trace_explore_start}-\ref{lst:line:get_possible_actions}). 

\begin{wrapfigure}{r}{0.57\textwidth}
\begin{minipage}[t]{0.57\textwidth}
\vspace{-0.7cm}
\begin{algorithm}[H]
  \small
  \begin{algorithmic}[1]
    \State {\bfseries Input: a message passing component $m$, a summary component $\summary$, a prediction component $\pred$, an initial state $G_0$, a transition system $\mathcal{T}$, and a specification $\phi : \phi_{in}\to \unreach(T)$} 
    \State {\bfseries Output: $\hold/\nothold/\unknown$}
    \Function{checkWithTraceEnumeration}{$m, f, g, G_0, \phi$}
    \State $stack \mapsto \{\texttt{nil}\}$
    \While {$\neg stack.\textsc{empty}()$}
        \State $t \mapsto stack.\textsc{pop}()$\label{lst:line:trace_pick}
        \State $r \mapsto \matchTrace(t, T)$\label{lst:line:trace_match}
        \If {$r = \match$}
            {\bf return} $\nothold/\unknown$
        \ElsIf {$r = \notmatch$}
             {\bf continue}\label{lst:line:trace_nomatch}
        \Else { $M, G, k \mapsto \phi_{in}, G_0, 0$}\label{lst:line:trace_explore_start}
            \For {$\node_i \text{ in } t$}\label{lst:line:for_loop_encode_start}
                \State $M \mapsto M \land \encodeState(m,\pred,\summary, G)$\label{lst:line:encode_state}
                \State $M \mapsto M \land \encodeAction(M, \node_i)$\label{lst:line:encode_action}
                \State $M \mapsto M \land \encodeTransition(\mathcal{T}, G, \node_i)$\label{lst:line:encode_transition}
                \State $G, k \mapsto \mathcal{T}(G, \node_i), k + 1$\label{lst:line:for_loop_encode_last}
            \EndFor
            \State $M \mapsto M \land \encodeState(m,\pred,\summary, G)$\label{lst:line:encode_state2}
            \State {$\Theta \mapsto \getFrontier(G)$}
            \State $Q \mapsto \checkDisjuncts(M, \Theta)$ \label{lst:line:get_possible_actions}
            \State $stack \mapsto stack \cup \{t::\node_i | \node_i \in Q\}$
        \EndIf
    \EndWhile
    \State {\bf return} $\hold$\label{lst:line:multi_step_hold}
    \EndFunction
    \Function{$\checkDisjuncts$}{$M, \Theta, G$}
    \State {$Q \mapsto \{\}$, $\phi(\node) \mapsto (\bigvee_{j\in\Theta}\node_j \geq \node)$}\label{lst:line:node_not_max}
    \For {$J \in G$}
    \State {$\phi_{out} \mapsto \bigwedge_{\node_i\in J\cap\Theta} \phi(\node\mapsto \node_i)$}
    \State {$Q\mapsto Q\cup \checkWithNodeAbstraction'(M, \phi_{out})$}
    \EndFor
    \State {\bf return} Q
    \EndFunction
  \end{algorithmic}
  \caption{Multi-step verification via trace enumeration.\label{alg:multi}}
\end{algorithm}
\vspace{-0.8cm}
\end{minipage}
\end{wrapfigure}
Note that an exact encoding of the transition system up to $t$ (Lines~\ref{lst:line:for_loop_encode_start}-\ref{lst:line:for_loop_encode_last}) involves the precise encoding of 1) the network at each time step (Line~\ref{lst:line:encode_state}); 2) the action taken at each time step (Line~\ref{lst:line:encode_action}); and 3) the updates of the feature vectors (Line~\ref{lst:line:encode_transition}).An incomplete encoding can be achieved by ignoring the first two components. This amounts to ignoring the previous trace and only encoding the network at the current step. We explore the runtime-precision trade-off of the complete encoding in our experimental section. The algorithm returns \hold\xspace if no traces from $T$ are matched during the enumeration. Notice that $T$ is only used in the $\match$ function for checking whether the current trace we are exploring is a ``bad'' trace. Therefore, in practice, instead of providing $T$ as a concrete set of traces, it is sufficient and relatively easier  to provide an implementation of the $\match$ function corresponding to the property.

The $\checkDisjuncts$ method at Line 18 can build on top of any single-step verification engine including, in particular, the procedure introduced in \secref{sec:abtraction}. One instantiation tailored to GNN-based job schedulers is described in \algoref{alg:multi}.
Our goal is to check whether an action $\node$ indexed by a set $\Theta$ can be scheduled. This can be formulated as checking whether the post-condition that $\node$ is not the maximum (Line~\ref{lst:line:node_not_max}) holds. We use the techniques introduced in \secref{sec:abtraction} to reason about candidate actions belonging to the same job DAG simultaneously. This is based on the observation that candidate actions corresponding to the same job often have similar verification results. Note that here we use a slightly modified version of \algoref{alg:node-abstract} where instead of returning $\hold/\nothold/\unknown$ we return all disjuncts in the post condition that does not hold. This amounts to modifying Line~\ref{lst:line:multi_step_hold} in \algoref{alg:multi} to continue the search instead of returning.

\begin{theorem}\label{thm:stop}
\algoref{alg:multi} terminates if $\ T$ has finite length traces.
\end{theorem}
\begin{proof}
Let $K$ be the maximum number of possible actions at a given time step and $R$ be the longest trace length in $T$. In the worst case, there are a finite number of traces ($K^R$) to check. This is because the algorithm: (a) does not check the same trace twice (this can be proved by induction on the length of the partial trace $t$); and (b) does not expand any $t$ whose prefix does not match a trace in $T$ (Line~\ref{lst:line:trace_nomatch} of Algorithm~\ref{alg:multi}). 
\end{proof}

\begin{theorem}[Soundness]\label{thm:sound}
If the encodings (Lines~\ref{lst:line:encode_state},\ref{lst:line:encode_action},\ref{lst:line:encode_transition}, \ref{lst:line:encode_state2}) and $\checkDisjuncts$ are sound,  
then \algoref{alg:multi} is sound. 
\end{theorem}
\begin{proof}
The assumptions guarantee that $Q$ is a super-set of the actual feasible actions. By induction on the length of the trace added to $stack$, we can prove that the traces added to $stack$ are a super-set of the actual reachable set of traces. Therefore, if no trace added to the stack matches any trace in $T$, no actual reachable trace can match any trace in $T$.  
\end{proof}


\paragraph{Proof-transfer encoding.} In general, if there are changes in node features or graph structures, the message passing would result in different node embeddings. A na\"ive complete encoding would re-encode message-passing (and subsequent GNN components) for every step. This quickly becomes too expensive as the number of time steps increases. However, taking a closer look at the message passing scheme, we observe that the effect of the graph structure/feature updates on the message passing is local to the disconnected component where the updates occur. This means that we only need to re-encode disconnected components of the graphs that are updated. In the case of Decima, we observe that between every scheduling event (invokation of the GNN agent), only a small subset of the job DAGs are updated. This results in significant savings in the length of the encoding and runtime as demonstrated by our experimental results.

\section{Specifications for job scheduler}
\label{sec:spec}

In this section we define the properties we verify for GNN-based schedulers like Decima. We emphasize that our framework allows the user to specify a rich set of verification properties. Here, we focus on two formulations of the strategy-proofness properties to demonstrate the capabilities of our method. We choose to study strategy-proofness as it is not only important in practice but also representative of the general form of specifications that our framework can handle.
\extVersion{We additionally investigate the data-locality property of the scheduler in App.~\ref{app:locality}, which was identified as one of the critical factors in job scheduler performance~\cite{locality}.}

Strategy-proofness is a desirable property of schedulers that intuitively means: ``a user cannot benefit by mis-representing their need.''  For example, we expect that the user cannot get their jobs scheduled earlier by requiring more resources for them. If this basic property does not hold, malicious users can mislead the system into stalling all but their jobs.  Interestingly, this property holds for simple schedulers such as FIFO (first-in-first-out) and CMMF (Constrained Max-Min Fairness)~\cite{shenker2013choosy}. However, due to the non-interpretable nature of the GNN-based scheduler, strategy-proofness cannot be guaranteed by construction. 

\begin{definition}[Single-step strategy-proofness\label{def:single-step-sp}]
Given an initial job profile $G = (A, X)$ containing K jobs $G_1, \ldots, G_K$, suppose the scheduler picks a node from job $G_k$. For each node feature vector $\vect{\feat_i}$, let $\feat_{id}$ and $\feat_{it}$ denote the entries of estimated total duration and the number of tasks, respectively. Let $G_a\in G$ be a job other than $G_k$ (e.g., the job of an adversarial user). Let $C$ and $C_a$ denote the frontier nodes in $G$ and $G_a$, respectively. The job scheduler is strategy-proof with respect to $G$ and $G_a$, where $a\neq k$, if $\forall G' = (A, X')$,
\begin{align*}
\bigwedge_{\node_i \in C_a} (\feat'_{id} \in [\feat_{id}, \scale_d \feat_{id}] \tikzmark{top2}\\
\land \feat'_{it} \in [\feat_{it}, \scale_t \feat_{it}] \tikzmark{right2}\\
\land \frac{\feat'_{id}}{\feat'_{it}} \geq \frac{\feat_{id}}{\feat_{it}})\tikzmark{bot2}\\ 
\to \bigwedge_{\node_i\in C_a} \Big(\neg\big(\bigwedge_{\node_j\in C\backslash C_a} \out'_i > \out'_j\big)\Big)
\end{align*}
\AddNote{top2}{bot2}{right2}{\scriptsize $:=\phi^{sp}_{in}(X')$}
where $\scale_d$ and $\scale_t$ are scalars (> 1).
\end{definition}

Intuitively, the pre-condition specifies that the owner of the adversarial job $G_a$ can increase all the features related to job utilization as well as the average task duration (implied by the third constraint) in the frontier nodes of their job. The $\alpha$ parameters define the level of perturbation that the adversary is allowed. 
The post condition states that none of the frontier nodes in $G_a$ can be scheduled, thus implying strategy-proofness. Note that the inner constraint in the post condition $\bigwedge_{j\in C\backslash C_a} \out_i > \out_j$ is a simple post condition (by De Morgan's law), thus the post condition has the form described in Eq.~\ref{eq:dnf}. Also note that the strategy-proofness property is very different from the adversarial robustness property~\cite{szegedy2013intriguing}, which states that the neural network's decision does not change in response to small perturbation in the input feature. In contrast, strategy-proofness allows the network's decision to change as we increase the input features of $G_a$. The property holds as long as no node from $G_a$ is scheduled next (i.e., the malicious user cannot benefit). 

\begin{definition}[$T$-step strategy-proofness]
Given a concrete initial job profile $G = (A, X)$ containing K jobs $G_1, \ldots, G_K$, let $G_a\in G$ be a job that is not scheduled within $T$ steps starting from $G$. The job scheduler is $T$-step strategy-proof with respect to $G$ and $G_a$, where $a\neq k$, if $\forall G' = (A, X')$,
\begin{align*}
\phi^{sp}_{in}(X') \to \unreach(\{t \mid |t|\leq T \land (\exists \node \in C_a, s.t. \node \in t)\})
\end{align*}
\end{definition}

The pre-condition is the same as in Def.~\ref{def:single-step-sp}. Intuitively, the property states that for a job $G_a$ that is not scheduled in $T$ step in the original trace starting from $G$, the owner of the job cannot get the job to be scheduled earlier by lying about the amount of work in the job. Note that when $T$ is equal to 1, the definition is equivalent to the single-step strategy-proofness property in Def.~\ref{def:single-step-sp}.

Without a verification engine, we could only obtain \emph{empirical} guarantees about both properties by checking a finite number of job profiles within the range specified by the pre-condition. However, our framework allows us to check \emph{all} job profiles in that range, thus obtaining \emph{formal} guarantees.

\paragraph{Local vs. global properties} Similar to previous work in neural network verification, the strategy-proofness properties here are defined locally, i.e., with respect to concrete initial job profiles. It is possible to define and verify a global version of the strategy-proofness property, for a set of initial job profiles with the same graph structure using our framework\extVersion{ (see App.~\ref{app:global-sp})}. However, checking this property requires a second copy of the GNN-encoding, thus doubling the size of the verification problem. Taking a step even further, it is possible to verify that the strategy-proofness holds for all job profiles with less than $N$ nodes, by checking the property with our verification engine for each unique job profile with less than $N$ nodes (there are finitely many of them). Abstraction techniques for reasoning about job profiles with different graph structures as a whole are likely needed to improve the verification time. We leave the verification of global properties as future work.  

\section{Experimental Evaluation}
\label{sec:experiments}

We implemented the proposed techniques in a tool called \sys and
performed an experimental evaluation by using \sys to check whether the single and multi-step strategy-proofness properties as described in the previous section hold on Decima~\cite{decima}---a state-of-the-art GNN-based job scheduler. 
There are other GNN-based schedulers available (e.g., \cite{park2021learning,sun2021deepweave}). We choose Decima as the representative as it is by far the most popular and influential. We note that our techniques are general and applicable to other GNN-based schedulers and properties as discussed above.
Feedback from \sys can be used by the system developer to adjust their schedulers to balance different user expectations. 

\subsection{Implementation}

Our implementation of \sys is three-fold, including:
\begin{enumerate}
\item  A \textbf{GNN-based scheduler module} that takes the graph structure of the job profile and the architecture of the GNN agent, and converts them into an internal representation of a feed-forward neural network (with residual connections) with the node features as inputs and node predictions as outputs. 
As the front-end of \sys,
the module also contains API calls to define pre-conditions and post-conditions of the forms described in Subsecs.~\ref{subsec:spec} and ~\ref{prelim:gnn-verify}.
\item  A \textbf{single-step verification engine} which contains a generic implementation of the algorithms introduced in Secs.~\ref{sec:backward} and \ref{sec:abtraction}. For forward-backward analysis, we use DeepPoly as the abstract domain, and use the linear approximation proposed in~\cite{planet} extended to Leaky ReLU for the LP encoding. It is worth noting that unlike the original DeepPoly implementation~\cite{deeppoly}, our forward-backward analysis handles feed-forward neural networks with arbitrary residual connections. This generality is needed even for verifying different properties of the same GNN because the graph structures of the initial state affect the order of message passing and yield neural networks with different architectures. 
The node-abstraction implementation applies to GNN-based node prediction models generally.
\item A \textbf{multi-step verification engine} that contains an implementation of the trace enumeration procedure described in Sec.~\ref{sec:multi}, which repeatedly invokes the single-step verification engine. We support both complete and incomplete multi-step encodings.
\end{enumerate}

\subsection{Experimental setup} 
We use the same GNN-architecture and training configurations as introduced in the original work~\cite{decima}, and obtain similar performance results as in the original work. \extVersion{Training details can be found in \appref{app:training}.} The trained network has Leaky ReLU as activation function. Unlike the traditional neural network verification setting, where the neural network size (e.g., number of activations) is fixed, the GNN size depends on the sizes of the input graphs, which we specify later. 
All experiments are run on a cluster equipped with Intel Xeon E5-2637 v4 CPUs running Ubuntu 16.04. Each benchmark is run with 32 threads and 128GB memory. For single-step verification, each benchmark run is given a 1-hour wall-clock timeout. For multi-step verification, the wall-clock timeout is set to 2 hours.

\subsection{Single-step verification}\label{subsec:single-step-experiment}
\sloppy
We first evaluate \sys on single-step verification benchmarks. The main question we pose here is: Does \sys scale to large GNN-based schedulers? 
To answer this question, we perform an extensive evaluation of all the techniques we proposed. 
Our results demonstrate a significant performance gain over baselines based on state-of-the-art verifiers.

\paragraph{Benchmarks.}
We evaluate our proposed techniques on verifying the single-step strategy-proofness property as introduced in \secref{sec:spec}. The scalars $\alpha$'s (see Def.~\ref{def:single-step-sp}) are set to 20, meaning the owner of the adversarial job can increase the estimated total duration and number of tasks by at most 20 times for any frontier nodes in the job. We consider job profiles that contain either 5 or 10 jobs, which yields a median of 5845 and 10997 activations (Leaky ReLU) in the encoding respectively. After unrolling, the network has 140 layers (treating affine and activation as separate layers). We heuristically select job profile and adversarial jobs ($G_a$'s) from Decima's test beds that would likely result in challenging verification benchmarks following the steps below:
\begin{itemize}
    \item Sample initial job states with 5 and 10 jobs from Decima's native test bed\footnote{\url{https://github.com/hongzimao/decima-sim/}} using random seeds 0-24. 
    \item For each of the initial job states, we run the simulation environment until 1/3 of the nodes are scheduled. At each step, we rank the jobs by the sum of the scores of the frontier nodes in decreasing order. We record steps where the total scores of frontier nodes in the top job and that in the second top job are close (<0.9). 
    \item We use strategy-proofness properties defined on states corresponding to those steps as the benchmarks for single step verification with $G_a$ being the second top job because these are vulnerable states that make for challenging verification benchmarks.
\end{itemize}
The resulting initial job profiles are sparse graphs, with each job containing on average $9.2 \pm 4.2$  nodes and $8.5 \pm 4.4$ edges. 

\paragraph{Configurations}
To evaluate our proposed techniques, we consider 4 different configurations: 
1) \base first tries to solve the problem with forward abstract interpretation and falls back to a complete MILP encoding with DeepPoly bounds via the Gurobi optimizer;
2) \backOnce performs the forward-backward analysis for one iteration (forward, backward, and forward again) and falls back to Gurobi;
3) \backConverge is the same as \backOnce except that it performs the forward-backward analysis repeatedly until the stopping condition described in \secref{sec:backward} is met, and
4) \abstBackConverge runs \algoref{alg:node-abstract} on top of \backConverge.
We note that $\base$ is equivalent to ERAN~\cite{deeppoly} with its optimal configuration for complete verification. 
We did not compare with off-the-shelf verification tools~\cite{prima,kpoly,bcrown,scaling,marabou}, because to our knowledge none can handle the complex architecture of GNNs without significant implementation overhead.

We first evaluate the first three configurations on the full benchmark set. The result is shown in \tabref{tab:backward}. We observe that the two configurations that perform the forward-backward analysis solve significantly more benchmarks than $\base$, which only performs a forward pass, with a gain of 77\% and 8\% of solved instances for graphs with 5 jobs and 10 jobs, respectively. On the other hand, we observe an incremental gain from performing the forward-backward analysis for 1 iteration to performing it to convergence.  While we solve 5 more benchmarks, performing forward-backward analysis until the stopping condition might lead to a non-negligible runtime overhead. For example, the average time to solve a 10-job benchmark increases from 206 seconds to 324 seconds (i.e., an overhead of ~2 minutes). 

\begin{table}[h!]
\caption{Instances solved by different configurations and their runtime (in seconds) on
  \emph{solved} instances (i.e., runtime for timed-out instances is not included). \label{tab:backward}}
\setlength\tabcolsep{10pt}
\centering		
\sffamily
\begin{tabular}{ccccccc}
  \toprule
\# jobs (\# bench.) 
& \multicolumn{2}{c}{\base}
& \multicolumn{2}{c}{\backOnce}
& \multicolumn{2}{c}{\backConverge}\\
\cmidrule(lr){2-3} \cmidrule(lr){4-5} \cmidrule(lr){6-7}
& Solved & Time & Solved & Time & Solved & Time \\
5 (66) & 26 & 32742 & 46 & 31147 & \noindent\textbf{49} & 38080 \\
10 (232) & 207 & 45662 & 224 & 46201 & \noindent\textbf{226} & 73295 \\
\bottomrule
\end{tabular}
\end{table}

The cactus plot in Fig.~\ref{fig:cactus} sheds more light on the pay-off of the abstraction refinement scheme. While \base can solve certain easy instances faster than \backOnce and \backConverge, the benefit of the proposed forward-backward analysis becomes evident once the time limit surpasses 100 seconds. On the other hand, \backConverge starts to overtake \backOnce when the time limit is above 1000 seconds. To further compare the long-term behaviors of \backOnce and \backConverge, we run them on the unsolved benchmarks in Tab.~\ref{tab:backward} (there are 23 of them) using a longer timeout (2 hours). \backOnce is able to solve 1 of the 23 benchmarks while \backConverge is able to solve 4. Based on these results, we recommend to use \backConverge when computational resources are not a concern. 

Among the 298 verification queries, 257 are proved, 20 are disproved (with counter-examples), and 28 are unknown. This suggests that the current version of Decima is not always strategy-proof, and adjustments in the training algorithm are potentially needed to guarantee it without compromising performance too much.

To evaluate the node abstraction scheme, we focus on the subset of benchmarks where there are at \emph{least 4 frontier nodes} in the adversarial job. This means that the number of disjuncts is at least 4. We run \abstBackConverge on this subset of benchmarks and compare it with $\backConverge$. As shown in \tabref{tab:abstraction}, while only 1 more instance is solved due to the node abstraction, the runtime is reduced significantly. In particular, while both configurations solve the same number of 10-job benchmarks within the time limit, \abstBackConverge solves them with a runtime reduction of 51.2\%. These results clearly demonstrate that using abstract nodes leads to significant computational saving if there are multiple frontier nodes available for scheduling.

\begin{figure}[t]
\centering
\begin{minipage}{.47\textwidth}
\centering
\vspace{-1mm}
\includegraphics[width=0.99\textwidth]{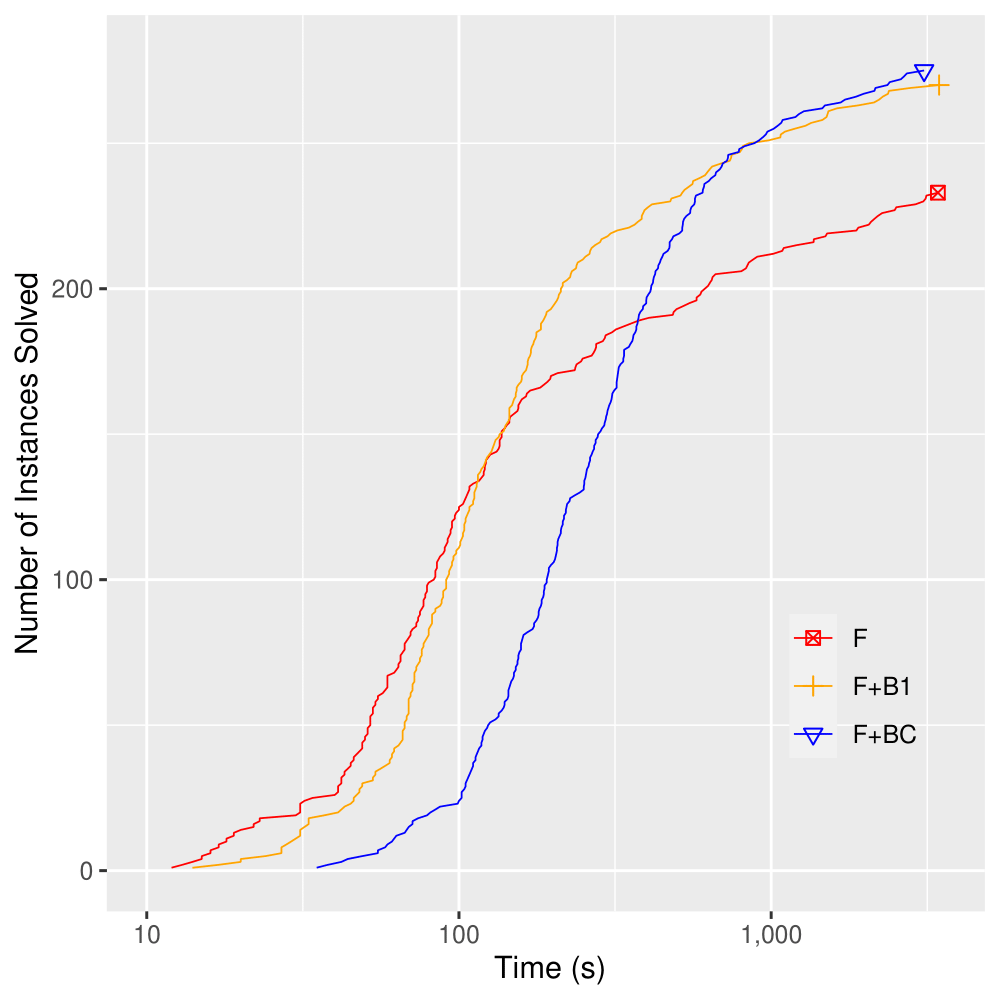}
\vspace{-4.5mm}
\caption{Cactus plot of the three configurations on the full benchmark set.}
\label{fig:cactus}
\end{minipage}%
\hspace{2mm}
\begin{minipage}{.48\textwidth}
\centering
\includegraphics[width=\textwidth]{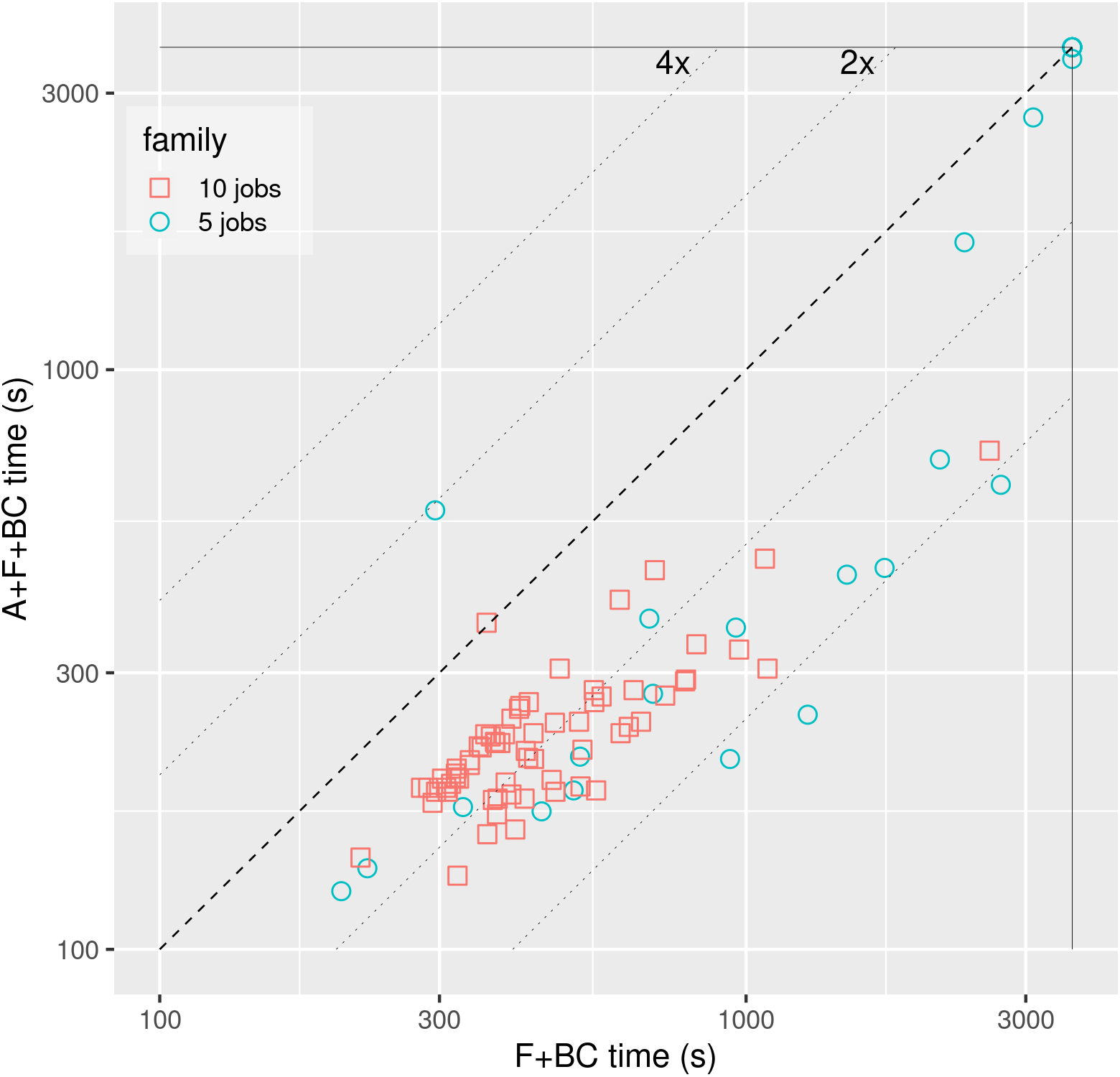}
\caption{Runtime of \abstBackConverge and \backConverge on benchmarks with at least 4 disjuncts.}
\label{fig:abstraction-scatter}
\end{minipage}
\end{figure}

\begin{table}[h]
\caption{Instances solved by different configurations and their runtime (in seconds) on solved instances. \label{tab:abstraction}}
\setlength\tabcolsep{12pt}
\centering		
\sffamily
\begin{tabular}{ccccc}
  \toprule
\# jobs (\# bench.) 
& \multicolumn{2}{c}{\backConverge}
& \multicolumn{2}{c}{\abstBackConverge}\\
\cmidrule(lr){2-3} \cmidrule(lr){4-5}
& Solved & Time & Solved & Time \\
5 (26) & 18 & 20600 & 19 & 13112 \\
10 (66) & 66 & 33061 & 66 & 16028 \\
\bottomrule
\end{tabular}
\end{table}

The scatter plot in Fig.~\ref{fig:abstraction-scatter} shows the concrete run time of the two configurations on the benchmarks. The node abstraction scheme brings around ~2-4x speed-up in a majority of the cases. We notice there are two cases where the verification time is not improved (or even becomes worse). In those cases, the attempts in Alg.~\ref{alg:node-abstract} to reason about multiple disjuncts simultaneously keep failing, and most disjuncts end up being analyzed individually. This suggests that additional performance gain could be potentially achieved if a more sophisticated heuristic to pick which node to remove from the equivalence class (see Sec.~\ref{sec:abtraction}) is used.

\subsection{Evaluating forward-backward analysis in isolation}\label{subsec:additional evaluation}

In this section, we further evaluate the effectiveness of the forward-backward analysis as a stand-alone technique. In particular, we pose two research questions: 
\begin{enumerate}
 \item Is the forward-backward analysis effective as a stand-alone technique (without falling back to a complete solver)? [Yes]
 \item Is the forward-backward analysis useful on other verification benchmarks such as adversarial robustness properties on image classifiers? [Yes]
\end{enumerate}

\paragraph{Benchmarks.}
We consider the same initial job profiles as described in Sec.~\ref{subsec:single-step-experiment}. The specification is also the same except that the scalars $\alpha$'s are set to 8 instead of 20. We choose this value of $\alpha$ because for larger values abstract interpretation alone usually fails to prove the property without invoking the complete solver and for smaller values (<1.5) forward abstract interpretation alone can prove most properties. 
Additionally, we train two classifiers, \mnistA and \mnistB, on the MNIST dataset~\cite{mnist}. Both are PGD-trained, fully-connected feed-forward, and using Leaky ReLU activations. \mnistA has 5 hidden layers with 100 neurons per layer. \mnistB has 8 hidden layers with 100 neurons per layer. We consider standard local $l_{\infty}$ adversarial robustness properties on the first 100 correctly classified test images. The perturbation bound is set to 0.02 (the inputs are normalized between 0 and 1).

\paragraph{Configurations} We consider three configurations,  \base', \backOnce', and \backConverge', which are the same as their counterparts in Sec.~\ref{subsec:single-step-experiment}, except that the former do not fall back to complete solvers and instead return \texttt{UNKNOWN} if abstraction interpretation fails to prove the property.

\begin{table}[h!]
\caption{Instances solved by different configurations and their runtime (in seconds) on
  \emph{solved} instances (i.e., runtime for timed-out instances is not included). \label{tab:additional}}
\setlength\tabcolsep{10pt}
\centering		
\sffamily
\begin{tabular}{ccccccc}
  \toprule
Benchmark (\#) 
& \multicolumn{2}{c}{\base'}
& \multicolumn{2}{c}{\backOnce'}
& \multicolumn{2}{c}{\backConverge'}\\
\cmidrule(lr){2-3} \cmidrule(lr){4-5} \cmidrule(lr){6-7}
& Solved & Time & Solved & Time & Solved & Time \\
SP 5 jobs (66) & 0 & 0 & 22 & 259 & \noindent\textbf{25} & 440 \\
SP 10 jobs (232) & 0 & 0 & 83 & 3195 & \noindent\textbf{94} & 4815 \\
\cmidrule(lr){1-7}
\mnistA (100) & 31 & 29 & 52 & 72 & \noindent\textbf{56} & 135 \\
\mnistB (100) & 24 & 37 & 43 & 100 & \noindent\textbf{46} & 175 \\
\bottomrule
\end{tabular}
\end{table}

\paragraph{Experiments} The evaluation results of the three configurations on the aforementioned benchmarks are shown in Table~\ref{tab:additional}. On the verification queries over GNNs, forward abstract interpretation alone (\base') is not able to solve any benchmarks, while the two configurations that perform backward abstraction refinement can solve a significant number of the benchmarks. This shows the benefits of forward-backward analysis as a stand-alone technique in improving verification precision and scalability. It is also worth noting that the effect of performing forward- and backward- analyses multiple times is more evident in this setting, where \backConverge' solves $13.3\%$ ($\frac{25 + 94}{22 + 83}$) more than \backOnce'.

On the adversarial robustness benchmarks, the forward-backward analysis also significantly boosts the verification precision. In particular, \backConverge' solves $85\%$ more than \base'. This confirms that the forward-backward abstract interpretation can be useful beyond GNN verification. 

\subsection{Multi-step verification\label{subsec:multi-step-experiment}}
We now turn to multi-step verification. Here we ask:
\begin{enumerate}
 \item Complete vs incomplete encodings. Is a complete encoding crucial to proving verification queries (due to imprecision of an incomplete encoding)? [Yes]
 \item Does the proof-transfer encoding speed up complete verification? [Yes]
\end{enumerate}

\paragraph{Benchmarks}
We evaluate our proposed techniques on verifying the $T$-step strategy-proofness as introduced in \secref{sec:spec}. We choose $T=5$, which requires enumeration of all possible traces of length 5, and $\alpha = 10$. Initial job profiles are heuristically selected following the steps below with the rationale of identifying challenging benchmarks:
\begin{itemize}
    \item Sample initial job states with 5 jobs from Decima's native test bed using random seeds 0-99. 
    \item For each initial job state, we run the simulation environment until 1/3 of the nodes are scheduled. At each step, if the score of the top node in the second top job is close to that of the top node in the top job ($< 0.9$),  we use the single-step verification engine to check whether from the pre-condition of the multi-step strategy-proofness properties, the scheduler can choose \emph{multiple} nodes as the next action. 
    \item For each step $t$ satisfying this condition, we add the job state at $t$, $t-2$, and $t-4$ to the multi-step benchmarks. The resulting benchmarks are guaranteed to have multiple possible traces from the initial set, thus making for more challenging benchmarks.
\end{itemize}
\extVersion{Additional evaluation result on the data-locality property~\cite{locality} can be found in App.~\ref{app:locality}.}


\begin{wrapfigure}{r}{0.64\textwidth}
\vspace{-5mm}
    \includegraphics[width=0.64\textwidth]{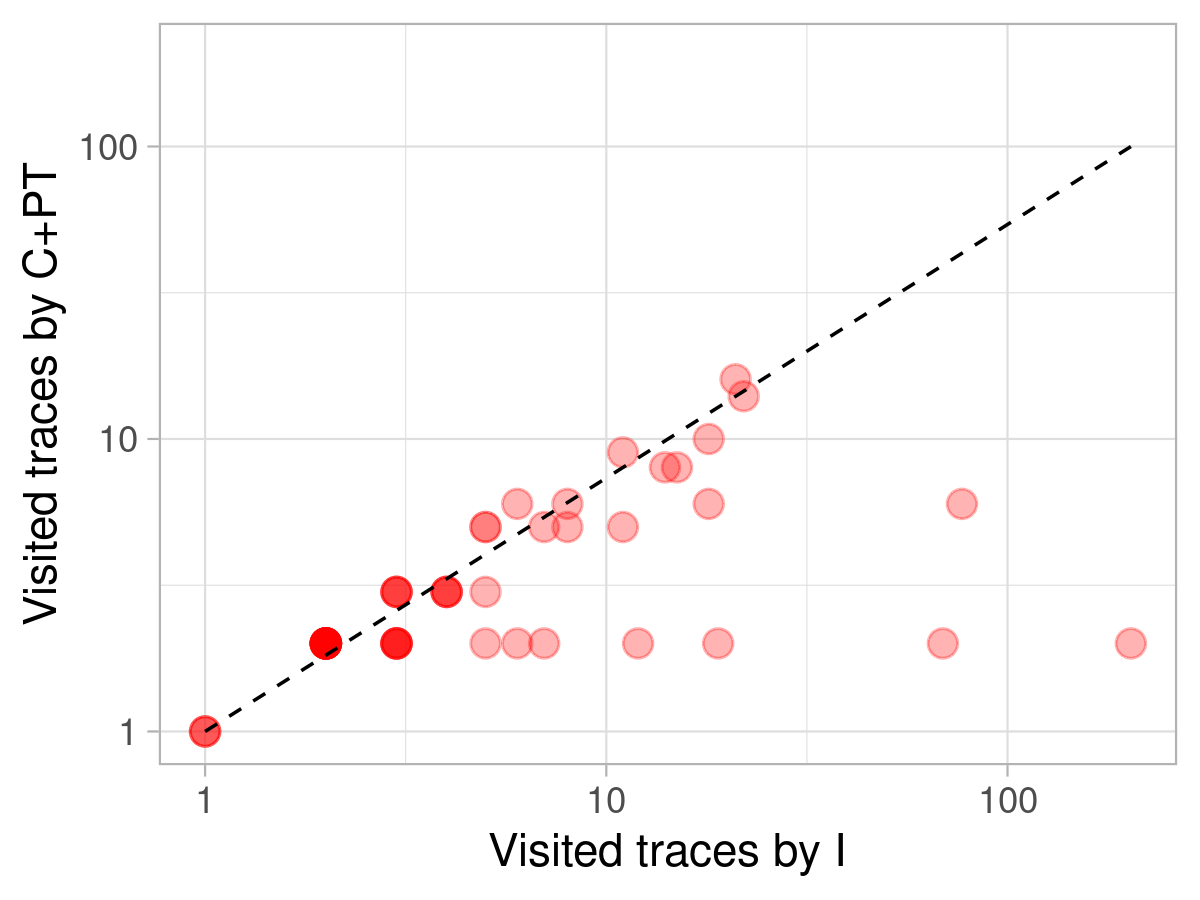}
  \vspace{-9mm}
  \caption{The number of visited traces by \singleStep and \completePT}
      \label{fig:ivscpt}
      \vspace{-4mm}
\end{wrapfigure}
\paragraph{Experiment}
We first evaluate the performance of the three configurations on the task of trace enumeration for 5 steps starting from the pre-conditions specified in strategy-proofness. 
\paragraph{Configurations}
We consider three configurations of \algoref{alg:multi}:
1) \singleStep does a sound but incomplete encoding of the state as described in \secref{sec:multi};
2) \completeNaive encodes the states and the state transitions precisely with a na\"ive unrolling;
3) \completePT also encodes the states and the state transitions precisely but with the proof-transfer encoding described in \secref{sec:multi}.

We found that \singleStep and \completePT both terminate on all benchmarks within 2 hours, while \completeNaive timed out on 15 of the 55 benchmarks. The average runtime of \singleStep, \completePT, and \completeNaive are respectively 350, 1129, and 3582 seconds (treating timed out instances as having runtime 7200 seconds).
To understand these results, we recall that all three configurations are based on trace enumeration and pruning. Both \completeNaive and \completePT encode the system precisely and therefore do not explore spurious (unrealizable) traces. However, as our experiment shows, the proof-transfer encoding results in significant speedup.

\figref{fig:ivscpt} shows the number of enumerated traces by \singleStep and \completePT on each benchmark. While \completePT enumerates the exact set of feasible traces, \singleStep enumerates a super-set of feasible traces due to its incomplete encoding. We observe that \singleStep includes spurious (i.e., infeasible) traces on 30 out of the 55 benchmarks. The difference in the number of visited traces can get quite large. For instance, for a certain initial condition (bottom right of \figref{fig:ivscpt}), there are only two feasible traces but \singleStep visits a total of 203 traces. The results emphasize the importance of a precise encoding, as exploring a large number of spurious traces can be computationally prohibitive and might prevent a solver from proving properties, as we see next. 

We now turn to verifying the multi-step strategy-proofness properties. To understand whether a precise encoding has benefits over an incomplete encoding, we focus on the 30 initial states where there is a difference in the enumerated traces between the complete and incomplete encodings. The result is shown in \tabref{tab:multi}. Due to the incomplete nature of \singleStep, not only is it unable to generate actual counter-examples, it also proves less properties than a complete procedure like \completePT. In contrast, \completePT solved all queries, and is able to prove 3 more specifications than \singleStep. This illustrates the benefit of a complete encoding. 

\begin{table}[h]
\caption{ \singleStep vs. \completePT on the multi-step strategy-proofness benchmarks. We show the number of instances that are proved and disproved respectively, as well as the runtime (s) on \emph{solved} instances. \label{tab:multi}}
\setlength\tabcolsep{5pt}
\centering		
\sffamily
\begin{tabular}{ccccccc}
  \toprule
Prop. (\# bench.)
& \multicolumn{3}{c}{\singleStep}
& \multicolumn{3}{c}{\completePT}\\
\cmidrule(lr){2-4} \cmidrule(lr){5-7}
& Proved & Disproved & Time & Proved & Disproved & Time \\
SP (30) & 19 & 0 & 14704 & \bf{22} & \bf{8} & 46099 \\
\bottomrule
\end{tabular}
\end{table}
\section{Related Work}
\label{sec:rel}

Most state-of-the-art neural network verifiers perform forward analysis, but a combination with backward analysis is under-explored. \cite{urban2020perfectly} proposes a procedure specific to fairness properties which uses a forward pre-analysis to partition the input region and a post-condition guided backward analysis to prove the properties for all activation patterns in the input region. \cite{yang2021improving} proposes a fundamentally different abstraction-refinement loop where the backward analysis iteratively refines the pre-condition for DeepPoly analysis. In contrast, we develop a general framework for forward-backward analysis on neural networks that can be instantiated with different abstract domains and prove theoretical results about soundness, monotonicity, and computational complexity of the framework. We also validate our approaches on more complex benchmarks compared with the aforementioned work.
Related to abstraction refinement,  \cite{frown,ryou2021scalable} use the post condition to refine the choice of slopes in forward over-approximations but do not consider the post condition as a hard constraint. \cite{singh2019boosting} combines forward abstract interpretation with MILP/LP solving for refinement but only considers the pre-condition, and the refinement is due to the precision gain of the MILP/LP solving.

Verification of RL-driven systems have also gained increasing attention recently~\cite{sun2019formal,amir2021towards}. Most recently, \cite{amir2021towards} explores the general use of techniques such as k-induction and invariant inference in those settings while treating the verification procedure as a black box. In contrast, we focus on a challenging and interesting setting of verifying GNN-based job schedulers, where the neural network architecture is more complex and much larger. This motivates the development of a set of new techniques to improve solver's scalability.

The verification of Graph Neural Networks is an important yet under-explored topic. Previous work on GNN-verification focuses on structural perturbations~\cite{wang2021certified,bojchevski2019certifiable} and robustness properties. In contrast, we focus on a rich set of properties emerging from the scheduling domain. \cite{wang2021certified} uses random smoothing and therefore gives probabilistic guarantees. On the other hand, \cite{bojchevski2019certifiable} consider a finite perturbation space (adding/removing finite number of edges) while we focus on infinite perturbation sets defined as linear constraints over the node features.

\section{Conclusion and future work}
\label{sec:concl}
In this work, we proposed the first verification framework for GNN-based job schedulers. This setting poses unique challenges due to deeper network architecture and richer specifications compared to those handled by existing neural network verifiers. We considered both single-step and multi-step verification and designed general methods for both that leverage abstractions, refinements, solvers, and proof transfer to experimentally achieve significantly better precision and speed than baselines. We believe that \sys can be used by system developers to check whether different user expectations are met and make adjustments if needed.
\sys can also be potentially integrated in the training loop of the GNN-based scheduler to guarantee by construction properties specified by the system designer. Similar approaches have been used to create stable neural network controllers in the robotics domain~\cite{dai2021lyapunov}.
We also believe that the verification benchmarks used in the experiments, which are different from the canonical adversarial robustness queries, would in themselves be a valuable contribution to the research community. 

We also note that our proposed techniques have different levels of generality. The forward-backward analysis applies to any feed-forward/convolutional/residual neural networks and any abstraction satisfying the conditions in Sec.~\ref{sec:backward}. The node abstraction scheme is generalizable to GNN-based node prediction tasks. The multi-step verification procedure is specific to job-scheduling but could potentially be extended to other RL-driven systems. Exploring the general effectiveness of the proposed techniques would be an interesting direction for future work. There are multiple other promising future directions. First, there is still room to improve the performance of \sys: e.g., using more precise abstraction domains or devising a specialized complete procedure that better leverages the problem structure. Second, it would be interesting to evaluate global properties of the GNN-scheduler, which (as discussed in Sec.~\ref{sec:spec}) presents additional research challenges.

\section*{Acknowledgments}

We thank the anonymous reviewers for their constructive feedback, and Guy Katz for some early discussion on forward-backward analysis. This work was conducted while the first author was an intern at VMWare Research.  It was also partially supported by NSF (RINGS \#2148583 and NSF-BSF \#1814369).

\newpage

\bibliography{bibli}

\ifextended
\newpage

\appendix

\section{Bound derivation with DeepPoly in Fig.~\ref{fig:backToyExample}\label{app:deeppoly}}

The forward abstract elements and bounds after the first forward abstract interpretation are shown in Fig.~\ref{fig:deeppoly}.
Each DeepPoly element contains four constraints, denoting the symbolic upper/lower bounds, and the concrete upper/lower bounds. The symbolic upper bounds of a Leaky ReLU output is the upper bound in the triangular linear relaxation~\cite{planet}. While the symbolic lower bounds are either $\alpha x_b$ or $x_b$ where $x_b$ is the input to the activation. An area heuristic is used to select which lower bound to use. The concrete bounds of a layer can be computed using a sequence of back-substitution steps described in \cite{deeppoly}.
\begin{figure*}[h]
    \centering
    \includegraphics[width=0.9\linewidth]{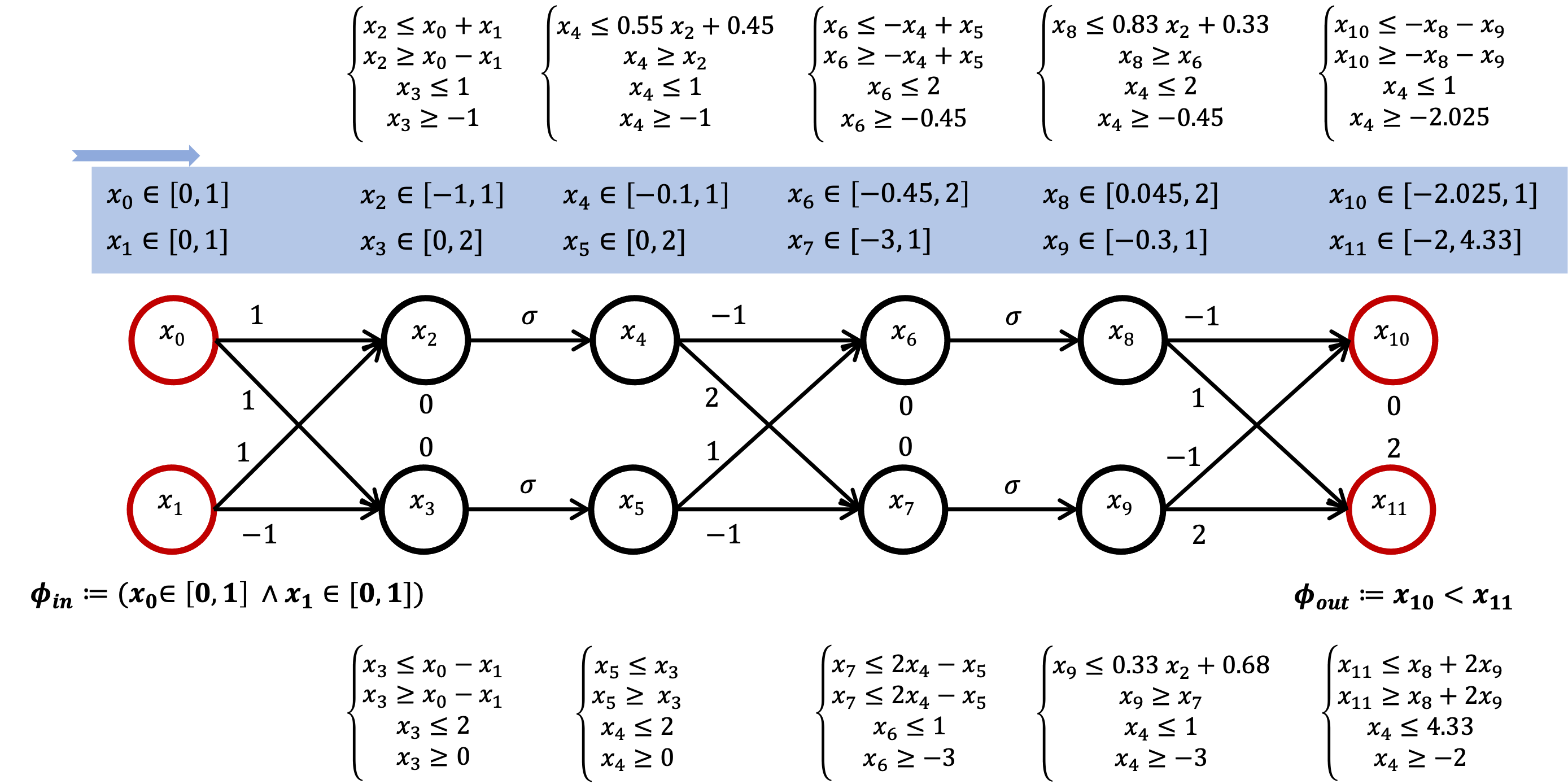}
    \caption{DeepPoly elements after the first forward analysis.}
    \label{fig:deeppoly}
\end{figure*}

\section{Training\label{app:training}}
We trained the Decima policy using the default parameters and the source code
published by Mao et al.~\cite{decima}. Specifically, we performed training on
TPC-H queries executed on the Alibaba production trace. To produce node, job,
and global embeddings, we used neural networks with six fully-connected layers
each, containing [16, 8, 8, 16, 8, 8] neurons in their layers,
respectively. Finally, two neural network with four fully-connected layers,
containing [32, 16, 8, 1] neurons respectively, mapped the embeddings to
actions--i.e., a node selected for scheduling and a number of executors to
assign. Training was performed using the REINFORCE policy-gradient
algorithm executed on 16 workers. We set the learning rate at
$1\times{}10^{-3}$ and used the ADAM optimizer. We ran training
for 10,000 episodes, when we observed that the average job completion-time has converged. Under this training configuration, Mao et al.\ found that Decima outperforms state-of-the-art scheduling algorithms.

\section{Feature vector of Decima \label{app:features}}
In Decima, each node in a job profile is described by a \emph{feature vector}. The feature vector of a node for Decima has the following 5 dimensions:
\begin{itemize}
    \item Number of executors in the job.
    \item Whether the current executor belongs to the job.
    \item Number of source executors.
    \item Total duration of the tasks currently in the job stage.
    \item number of tasks in the job stage.
\end{itemize}

The pre-condition of the strategy-proofness properties that we verify are linear constraints over the last two features of a group of nodes.

\section{Other properties. \label{app:spec}}
\subsection{Locality property.}\label{app:locality}
Data-locality is known to be crucial to scheduler performance~\cite{locality}.
To gain more insights into Decima's performance, we check whether the high performance is due to preservation of data locality:
when a node $\node_i$ is picked as an action (i.e. it is assigned an executor), Decima selects nodes in the same job when that executor is freed.  
We formulate a multi-step locality property for the multi-user single-executor setting: 

\begin{definition}[$K$-Locality] We say that the scheduler is $K$-local with respect to a set of initial job profiles $M=\{(A,X) \mid \phi_{in}(X)\}$, if $\forall G = (A, X), e\in E$,

\begin{align*}
\phi_{in}(X) \to \\
\unreach
\Big(
\big\{
t\mid \big(\{\cdot  \mid \node_c \}, \node_a,\node^1_{b},\ldots,\node^{K-1}_b\big)\subseteq t, \\ 
\differentJobs(\node_c, \node_a)\\
\land \exists \node^{j}_b, \differentJobs(\node^j_b,\node_a) \\
\land |\job(\node_a)@step(\node_a)| \geq K
\big\}
\Big)
\end{align*}
\label{def:locality}
\end{definition}

Intuitively, the post condition states that if the scheduler picks a node $\node_a$ from a certain job (assuming that either $\node_a$ is the first node in the trace, or the previous node belongs to a different job), it will schedule nodes from the \emph{same} job for at least $K$ consecutive steps if there are $K$ nodes left in that job. Note that depending on the scheduler characteristics and the choice of $K$, $K$-locality is not necessarily a desirable property. Here, as a way to understand Decima's performance, we check whether Decima is at least 2-local with respect to different input sets using Alg.~\ref{alg:multi}. In particular, we choose the pre-condition to be the same as that in the statrategy-proofness properties as described in Sec.~\ref{subsec:multi-step-experiment}. As a proof-of-concept, we check whether this property holds for 5 steps.

Again, we compare the effect of using the incomplete encoding (\singleStep) vs. the complete proof-transfer encoding (\completePT) on the 30 pre-conditions where the former produces spurious traces. Both configurations terminate within 2 hours for each benchmark. The result is shown in Table~\ref{tab:multi-local}. 

\begin{table}[h]
\setlength\tabcolsep{5pt}
\centering		
\sffamily
\begin{tabular}{ccccccc}
  \toprule
Prop. (\# bench.)
& \multicolumn{3}{c}{\singleStep}
& \multicolumn{3}{c}{\completePT}\\
\cmidrule(lr){2-4} \cmidrule(lr){5-7}
& Proved & Disproved & Time & Proved & Disproved & Time \\
Loc. (30) & 8 & 0 & 7912 & \bf{13} & \bf{17} & 28156 \\
\bottomrule
\end{tabular}
\caption{ \singleStep vs. \completePT on the 2-locality benchmarks. We show the number of instances that are proved and disproved respectively, as well as the runtime (s) on \emph{solved} instances. \label{tab:multi-local}}
\end{table}

We observe that \completePT again proves more properties than \singleStep, which is incomplete and unable to guarantee real counter-examples. Based on our results, the locality property does not seem to hold often. Therefore there are other factors to Decima's high performance. 
In the future, it would be interesting to look into other explanations and verify them with \sys.

\subsection{Global strategy-proofness property.\label{app:global-sp}}

\begin{definition}[Global single-step strategy-proofness\label{def:global-single-step-sp}]
Given a set of initial job profiles $M_A =\{(A,X) \mid \phi_{in}(X)\}$ with the same graph structure $A$. Suppose each job profile contains K jobs $G_1, \ldots, G_K$. For each node feature vector $\vect{\feat_i}$, let $\feat_{id}$ and $\feat_{it}$ denote the entries of estimated total duration and the number of tasks, respectively. Let $\node_{c}\in G_k$ be the node that is supposed to be scheduled next. Let $G_a\in G$ be a job other than $G_k$ (e.g., the job of an adversarial user). Let $C$ and $C_a$ denote the frontier nodes in $G$ and $G_a$, respectively. The job scheduler is strategy-proof with respect to $M_A$, $\node_c$ and $G_a$, if $\forall G = (A, X), G'=(A,X')$,
\begin{align*}
\Big(\phi_{in}(X) \land \bigwedge_{\node_i \in C\backslash \{\node_c\}} \out_{c} \geq \out_i \\
\bigwedge_{\node_i \in C_a} \big(\feat'_{id} \in [\feat_{id}, \scale_d \feat_{id}] \land \feat'_{it} \in [\feat_{it}, \scale_t \feat_{it}] \land \frac{\feat'_{id}}{\feat'_{it}} \geq \frac{\feat_{id}}{\feat_{it}}\big)\Big)\\ 
\to \bigwedge_{\node_i\in C_a} \Big(\neg\big(\bigwedge_{\node_j\in C\backslash C_a} \out'_i > \out'_j\big)\Big)
\end{align*}where $\scale_d$ and $\scale_t$ are scalars (> 1).
\end{definition}

Intuitively, the first constraint in the pre-condition specifies that $\node_c$ will be scheduled next for a job profile $G\in M_A$. The rest of the pre-condition is the same as the pre-condition in the local property (Def.~\ref{def:single-step-sp}), which specifies that $G'$ is a job profile resulting from increasing the job utilization features and the average task duration of $G$. The post condition also stays the same, stating that none of the frontier node in the adversarial job $G_a$ can be scheduled. In practice, it would be ideal to establish this property for each $\node_c \in C$ and $G_a\in \{G_1, \ldots, G_K\}\backslash \{G_k\}$, therefore guaranteeing that every $G\in  M_A$ is locally strategy-proof with respect to any potential adversarial job.

It is possible to verify this property with our framework, with two practical challenges. First, while for local strategy-proofness we only need to track the neural network's outputs corresponding to $G'$, in the global case, we also need to track its outputs corresponding to $G$. This means we need another copy of the neural network encoding, resulting in doubling of the problem size. Second, since $X$ is no longer constant, $\frac{\feat'_{id}}{\feat'_{it}} \geq \frac{\feat_{id}}{\feat_{it}}$is no longer linear. Therefore, we need to consider the linear over-approximation of the non-linear constraint. Several techniques (e.g., incremental linearization) for constructing this approximation exist in the literature.

\fi



\end{document}